\documentclass{article}

\usepackage{microtype}
\usepackage{graphicx}
\usepackage{booktabs} 

\usepackage{hyperref}

\usepackage[accepted]{icml2025}

\usepackage{amsmath}
\usepackage{amssymb}
\usepackage{mathtools}
\usepackage{amsthm}

\usepackage[capitalize,noabbrev,nameinlink]{cleveref}

\theoremstyle{plain}
\newtheorem{theorem}{Theorem}[section]

\theoremstyle{definition}

\theoremstyle{remark}

\usepackage[textsize=tiny]{todonotes}

\usepackage{multirow}

\usepackage{caption}
\usepackage{subcaption}

\usepackage[toc,page,header]{appendix}
\usepackage{minitoc}

\usepackage{enumitem} 

\usepackage{soul}

\usepackage{adjustbox}

\usepackage{xcolor}

\usepackage{enumitem}

\usepackage{optidef}

\DeclareMathOperator*{\argmax}{argmax} 

\usepackage{tikz}
\usetikzlibrary{matrix}

\usepackage{makecell}
\usepackage{color}

\usepackage{todonotes}
\usepackage{xargs}

\newcommandx{\DJJ}[2][1=]{%
  \todo[inline,
        linecolor=red,
        backgroundcolor=red!25,
        bordercolor=red,
        #1,
        size=\small]{DJ: #2}%
}
\newcommandx{\MM}[2][1=]{%
  \todo[inline,
        linecolor=orange,
        backgroundcolor=orange!25,
        bordercolor=orange,
        #1,
        size=\small]{MM: #2}%
}

\newcommand{\ourmethod}{\textsc{CAOTE}}

\usepackage{pifont}
\newcommand{\cmark}{\color{green}{\ding{51}}}%
\newcommand{\xmark}{\color{red}{\ding{55}}}%

\icmltitlerunning{\ourmethod{}: KV Cache Selection for LLMs via Attention Output Error-Based Token Eviction}

\begin{document}


\twocolumn[
\icmltitle{\ourmethod{}: KV Cache Selection for LLMs via Attention Output Error-Based Token Eviction}



\icmlsetsymbol{equal}{*}

\begin{icmlauthorlist}
\icmlauthor{Raghavv Goel}{comp}
\icmlauthor{Junyoung Park}{comp}
\icmlauthor{Mukul Gagrani}{comp}
\icmlauthor{Dalton Jones}{comp}
\icmlauthor{Matthew Morse}{comp}
\icmlauthor{Harper Langston}{comp}
\icmlauthor{Mingu Lee}{comp}
\icmlauthor{Chris Lott}{comp}
\end{icmlauthorlist}

\icmlaffiliation{comp}{Qualcomm AI Research is an initiative of Qualcomm Technologies, Inc.}

\icmlcorrespondingauthor{Raghavv Goel, Junyoung Park, Mingu Lee}{raghgoel,junpark,mingul@qti.qualcomm.com}

\icmlkeywords{Machine Learning, ICML}

\vskip 0.3in
]



\printAffiliationsAndNotice{}  

\begin{abstract}

While long context support of large language models has extended their abilities, it also incurs challenges in memory and compute which becomes crucial bottlenecks in resource-restricted devices. 
Token eviction, a widely adopted post-training methodology designed to alleviate the bottlenecks by evicting less important tokens from the cache, typically uses attention scores as proxy metrics for token importance.
However, one major limitation of attention score as a token-wise importance metrics is that it lacks the information about contribution of tokens to the attention output.
In this paper, we propose a simple eviction criterion based on the contribution of cached tokens to attention outputs. Our method, \textit{CAOTE}, optimizes for error due to token eviction, by seamlessly integrating attention scores and value vectors. This is the first method to use information from the value tokens on top of attention-based eviction scores in closed-form. Additionally, \textit{CAOTE} can act as a meta-heuristic method with flexible usage with any token eviction method. 
We show that \textit{CAOTE}, when combined with state-of-the-art attention score-based methods, always improves accuracies on the downstream task for L{\small LAMA3} and Q{\small WEN2.5} model families, indicating the importance of leveraging information from values during token eviction process. 

\end{abstract}
\icmlkeywords{token eviction, training-free long context}

\section{Introduction}
\label{section:introduction}


Large Language Models (LLMs) have demonstrated impressive capabilities across a wide range of natural language tasks, including text generation \cite{coenen2021wordcraft}, machine translation \cite{xiao2023survey}, and question answering \cite{robinson2022leveraging}. Many of these applications—such as retrieval-augmented generation (RAG), long-form document understanding \cite{liao2024doclayllm}, summarization \cite{zhang2024benchmarking}, and multi-turn dialogue systems \cite{thoppilan2022lamda}—require processing long input sequences, giving rise to the class of long-context LLMs.

A central challenge in long-context inference is the increased computational and memory overhead, particularly during the prefill and generation phases. This is primarily due to the quadratic complexity of self-attention and the growing memory footprint from storing activations over extended sequences. To mitigate quadratic attention computation, Key-Value (KV) caching has become a standard optimization, enabling faster inference by reusing previously computed key-value states \cite{pope2023efficiently}. However, in long-context settings, the memory consumed by the KV cache can surpass that of the model itself \cite{zhang2024h2o}, posing a significant bottleneck—especially for memory-constrained environments.

%
Recent efforts to manage KV cache memory can be broadly categorized into training-based approaches \cite{zhang2023adaptive, xiaoefficient} and post-training methods. This work focuses on the latter, which operate under fixed memory budgets and dynamically evict tokens from the cache. Early post-training methods such as H2O \cite{zhang2024h2o} retain tokens with the highest cumulative attention scores, while subsequent approaches refine this strategy using variants of attention-based metrics \cite{oren2024transformers, li2024snapkv, tang2024quest, qin2025cake}. These methods exploit the sparsity of attention—where a small subset of tokens disproportionately influence the output—to preserve only the most salient cached entries.

A key limitation of existing eviction methods is their reliance solely on attention scores, which reflect the alignment between query and key tokens. However, the output of the self-attention layer—the attention output—is a weighted combination of these scores and the corresponding value tokens. Since this output directly influences the hidden state and ultimately the model's predictions, ignoring the contribution of value tokens can lead to suboptimal eviction decisions.


In this paper, we propose \textit{CAOTE} (\textbf{C}ache Selection via \textbf{A}ttention \textbf{O}utput error-based \textbf{T}oken \textbf{E}viction), a post-training eviction method that directly computes the impact of each token on the attention output in closed form. Unlike prior approaches, \textit{CAOTE} integrates both key and value contributions to estimate the exact eviction error during generation as shown in Table \ref{tab:overview_table}. It is model-agnostic and can be applied as a meta-eviction strategy atop existing score-based methods. We also propose an efficient variant of \textit{CAOTE} suitable for deployment in constrained environments.

Empirically, \textit{CAOTE} consistently improves performance across a diverse set of benchmarks, including 16 tasks from LongBench, Needle-in-Haystack retrieval, and perplexity evaluations. When combined with recent eviction strategies \cite{zhang2024h2o, oren2024transformers, li2024snapkv}, \textit{CAOTE} yields substantial gains in accuracy and efficiency.


The paper is divided into the following sections: Section \ref{sec:background} discusses background for token eviction. Our main method is in Section \ref{sec:caote} which consists of \textit{CAOTE}. In Section \ref{sec:results}, we present experiments and results. Section  
\ref{section:related_work} consists related works and Section \ref{sec:conclusion} consists the conclusion.

\begin{table}[]
    \caption{Overview of recent token eviction methods compared to CAOTE based on components used during eviction.}
    \label{tab:overview_table}
    \centering
    \resizebox{\columnwidth}{!}{%
    \begin{tabular}{l|c|c|c}
         Method &  Keys & Values & Minimize eviction error
         \\
         \hline
         H2O & \cmark & \xmark & \xmark
         \\
         TOVA & \cmark & \xmark & \xmark
         \\
         SnapKV & \cmark & \xmark & \xmark
         \\
         X+CAOTE & \cmark & \cmark & \cmark
    \end{tabular}
}
\end{table}

\section{Background}
\label{sec:background}


Token eviction is a popular methodology \cite{xiaoefficient, oren2024transformers, zhang2024h2o} for decoder-only transformer inferences that prevents KV-Cache from growing linearly as token generation continues by preventing less important tokens from being cached. This has dual benefits; first, it limits memory consumption for the KV-cache and second, it reduces the computational complexity of the attention mechanism. Here, we consider the case of processing the input prompt block-wise in resource-restricted environments. In this case, token eviction can save memory and computation not only in the generation phase, but also in the prefill phase, leading to a shorter time-to-first-token which is especially beneficial when the input prompt is extremely long. 

With a sequence of hidden states $X^{l}=[x_{1}^{l}, \dots x_{t}^{l}] \in \mathbb{R}^{t \times d}$, the transformer block updated the hidden states as follows:
\begin{align}
    X^{l+1} = \Phi_{\text{TRANS}}^{l}(X^{l}) = \phi^{l}_{\text{FF}}\left(\phi^{l}_{\text{SA}}(X^{l})\right)
\end{align}

where, $x^{l}_{j}$ is the hidden state of token $j$, $\phi_{FF}^{l}$ denotes the feedforward layer, and $\phi_{SA}^{l}$ denotes the self-attention layer, superscript $l$ denotes the layer index. 

For brevity, we omit normalization layers and skip connections. 

\paragraph{Prompt prefill} Given hidden-states $X^{l} \in \mathbb{R}^{t \times d}$ of $t$ tokens, the self-attention layer process inputs as follows:
\begin{align}
X^{l}_{\text{attn}} = \phi_{\text{sa}}(Q^{l},K^{l}, V^{l})= \underbrace{\text{Softmax}(Q^{l}(K^{l})^\top)}_{A^{l}}V^{l}
\end{align}
where, $Q^{l}, K^{l}, V^{l} \in \mathbb{R}^{t\times d}$ and $A^{l} \in \mathbb{R}^{t \times t}$. 
Here, we omit output layer projection and multi-head extention for brevity.

\paragraph{Block-wise prompt prefill} Instead of processing all tokens at once (resulting in attention matrix: $A^{l} \in \mathbb{R}^{t \times t}$), we can process tokens in block-size $m$,  which also helps in evicting tokens in small blocks instead of larger chunks.
\begin{align}
    & X^{l}_{\text{attn}, t+1:t+m} = 
    \\
    & \underbrace{\text{Softmax}(Q^{l}_{t+1:t+m}[K^{l}_{:t},\mathbf{K^{l}_{t+1:t+m}}]^{T})}_{A^{l}\in \mathbb{R}^{m \times (t+m)}}[V^{l}_{:t},\mathbf{V^{l}_{t+1,(t+m)}}] \nonumber
\end{align}
where, the new token hidden states are $X_{t+1:t+m}^{l}$ which are projected to $Q^{l}_{t+1:t+m}, K^{l}_{t+1:t+m}, V^{l}_{t+1:t+m}$

\textbf{Generation}. In autoregressive generation a single token is generated at each iteration
\begin{align}
     X^{l}_{\text{attn}, t+1} = \underbrace{\text{Softmax}(Q^{l}_{t+1}[K^{l}_{:t},\mathbf{K^{l}_{t+1}}]^{T})}_{A^{l}\in \mathbb{R}^{1\times (t+1)}}[V^{l}_{:t},\mathbf{V^{l}_{t+1}}] 
\end{align}

For resource-constraint hardware, single-inference KV cache prefill for a large number of input tokens may cause out-of-memory error or slow throughput. On the other hand, combining block-wise prefill with token eviction after processing each block of prompt can resolve this issue and improve throughput~\cite{holmes2024deepspeed, agrawal2023sarathi}.
For a block-size $m$, when $b$ tokens are initially processed, the usage of memory and computation power can always be kept within budget constraints by processing the next $m$ tokens and evicting the next $m$ tokens. In this case, attention matrix has size: $A^{l} \in \mathbb{R}^{m \times (b+m)}$.


Recent eviction methods use variants of attention scores from $A^{l}$ for evicting tokens by using a function (or operator) to map $f_{\text{score}}(A^{l}): \mathbb{R}^{m \times (b+m)}\to \mathbb{R}^{b+m}$, where $f_{\text{score},j}$ is the retention score (or score) for token $j$, the top-$b$ tokens are retained based on the score: $\underset{b}{\argmax} \; f_{\text{score}}(A^{l})$, where $b$ is the budget (maximum tokens allowed per layer). Examples of score functions, for \textit{H2O}, $f_{\text{score},j}=\Sigma_{i=1}^{m}A_{i,j}$, and for  \textit{TOVA}, $f_{\text{score},j}=A^{l}_{-1,j}$. The process of token eviction follows intuitive steps as shown: computing scores for newly processed tokens, choosing top-$b$ tokens, computing attention output using the top-$b$ tokens' hidden-state, we show the steps below:
\begin{align}
    A^{l}_{b+m} = & \text{Softmax}(Q^{l}_{b+1:b+m}[K^{l}_{:b},\mathbf{K^{l}_{b+1:b+m}}]^{T})
    \\
    i_{1},\dots, i_{b} = & \underset{j\in\{1,\dots,b\}}{\argmax} f_{\text{score},j}(A^{l}_{b+m})
    \\
     X^{l}_{\text{attn}} =& \text{Softmax}(Q^{l}_{b+1:b+m}
     (K^{l}_{i_{1}:i_{b}})^{T}) V^{l}_{i_{1}:i_{b}}
\end{align}
where the key in bold are correspond to the new tokens' hidden states being inserted. In above equation we assume that no new query token was evicted for ease of notation. During generation, the flow remains same with $m=1$.

\section{CAOTE: KV Caching through Attention Output-Based Token Eviction}
\label{sec:caote}
Our method is developed based on two key insights: (i) existing token eviction policies primarily rely on attention scores derived from queries and keys, and (ii) attention output is a linear combination of values. 
We find that optimizing for eviction error is same as change in attention output due to eviction, which can be computed in closed-form for each token during generation and can be used as the eviction score (\textit{CAOTE} score).

\begin{figure*}
    \centering
    \includegraphics[width=1.0\linewidth]{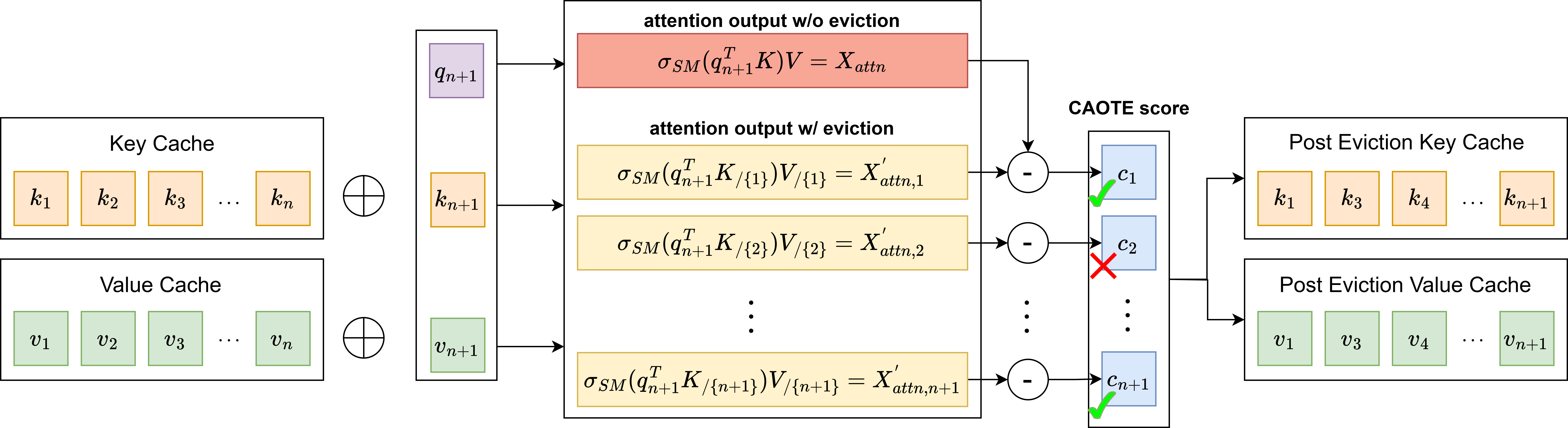}
    \caption{General flow of cache eviction when \textit{CAOTE} is integrated with existing cache eviction methods. We compute the impact of removal of each token to the attention output, this is same as eviction error (or \textit{CAOTE} score: $c_{1}, c_{2}, \dots c_{n+1}$). The token with the least impact is removed.
    }
    \label{fig:caote-block-diagram}
\end{figure*}

We first introduce \textit{CAOTE} 
in Subsection \ref{subsec:caote_derivation} \textcolor{black}{and how to compute eviction error in closed-form}. This is followed by a discussion of its meta-property, demonstrating its applicability with other score-based attention methods such as H2O \cite{zhang2024h2o} in Subsection \ref{subsec:meta_caote}. Finally, we propose an efficient approximation of \textit{CAOTE} in Subsection \ref{subsec:fast_caote}. The general workflow of \textit{CAOTE} is illustrated in Fig. \ref{fig:caote-block-diagram}, highlighting that the modifications to existing token eviction methods are minimal.


\subsection{CAOTE Score}
\label{subsec:caote_derivation}
The objective of our token eviction is to minimize eviction error: the change in attention output before and after eviction. We formulate eviction error for the generation scenario in which a single new token is inputted and therefore a single token needs to be evicted to maintain the budget $b$. Throughout the paper, we will use eviction error and \textit{CAOTE} score interchangeably.
 
 Given the attention scores of $b+1$ tokens $A=[\alpha_{1},\dots \alpha_{b+1}] \in \mathbb{R}^{1\times b+1}$ w.r.t. the last input token and the values: $V=[v_{1},\dots, v_{b+1}] \in \mathbb{R}^{d_{\text{head}} \times b+1}$, where $d_{\text{head}}$ is the head dimension. The \textit{CAOTE} score for token $j \in \{1, \dots, b+1\}$ is defined as (we ignore the layer and head dependence for simplicity).
\begin{equation}
    c_{j} = f^{\text{caote}}_j(A,V) 
    = \frac{\alpha_{j}}{1-\alpha_{j}}||\underbrace{VA^{T}}_{X_{\text{attn}}} - v_{j}||_{2}
    \label{eq:caote_score_based_on_tova}
\end{equation}

We proof that \textit{CAOTE} score is same as the eviction error. We define eviction error for token $j$ as the mean square error between attention output before and after eviction. Using the same setup as above:
\begin{align}
    e_{\text{eviction},j} = ||X_{\text{attn}} - X_{\text{attn},j}'||_{2}
    \label{eq:eviction_error}
\end{align}
where, $X_{\text{attn}}$ is attention output before eviction and $X_{\text{attn},j}'$ is attention output after eviction token $j$.
\begin{align}
    X_{\text{attn}} = & \, \alpha_{1}v_{1} + \alpha_{2}v_{2} + \dots \alpha_{b+1}v_{b+1} = VA^{T} \label{eq:attn_output_pre_eviction}
    \\
    X_{\text{attn},j}'= & \, \alpha_{1}'v_{1}+\dots\alpha_{j-1}'v_{j-1}
    \label{eq:attn_output_post_eviction_j}
    \\
     & + \alpha_{j+1}'v_{j+1}\dots \alpha_{b+1}'v_{b+1} \nonumber
\end{align}
where, $\alpha_{i}' \, \forall i\in \{1,\dots,j-1,j+1,\dots b+1\}$ in Eq. \eqref{eq:attn_output_post_eviction_j} is the post-eviction attention score to maintain the sum of the attention score property of sum equal to $1$. In the following we show the relation between the pre and post eviction attention score for token $i$ after the eviction of token $j$. 
\begin{theorem}
\label{thm:alpha_dash}
Given a new input token that exceeds the budget ($b$) by $1$. A token needs to be evicted. For any token $j$ being evicted, given the retention scores pre-eviction and post-eviction for any token $i \ne j$ as $\alpha_{i}$ and $\alpha_{i}'$ respectively, then the following relation holds:
\begin{equation}
    \alpha_{i}' = \frac{\alpha_{i}}{1-\alpha_{j}}
\end{equation}
\vspace{-4mm}
\end{theorem}
\begin{proof}
    Let the last input token has index $n$, then we define \begin{align}
        S \triangleq & \sum_{l=1}^{n}\exp(q_{n}^{T}k_{l})
        \\
        S'_{j} \triangleq & 
         S - \exp(q_{n}^{T}k_{j})
    \end{align}
    The retention score for token $i$ after evicting token $j$ is 
    \begin{align}
        \alpha_{i}' = & \frac{\exp(q_{n}^{T}k_{i})}{S_{j}'}
        =  \frac{\exp(q_{n}^{T}k_{i})}{S - \exp(q_{n}^{T}k_{j})}
        =   \frac{\alpha_{i}}{1 - \alpha_{j}}
      \label{eq:X_attn_dash_comp}
    \end{align}
\end{proof}

\begin{theorem}
Given a new input token that exceeds the budget ($b$) by $1$. A token needs to be evicted. For any token $j$ being evicted, the eviction error from Eq. \eqref{eq:eviction_error} and CAOTE score from Eq. \eqref{eq:caote_score_based_on_tova} are exactly same:
\begin{align}
    c_{j} = e_{\text{eviction},j}
\end{align}
\end{theorem}
\begin{proof}
    Using Theorem \ref{thm:alpha_dash}, we can rewrite post-eviction attention output from Eq. \eqref{eq:attn_output_post_eviction_j} 
\begin{align}
    X_{\text{attn},j}' =& \frac{1}{1-\alpha_{j}}\big(\alpha_{1}v_{1}+\dots 
    + \alpha_{j-1}v_{j-1}
    \\ 
    &+ \alpha_{j+1}v_{j+1}+\dots \alpha_{b+1}v_{b+1}\big) \nonumber
    \\
    =& \frac{1}{1-\alpha_{j}}\big(X_{\text{attn}} - \alpha_{j}v_{j}\big)
    \label{eq:x_attn_dash_j}
\end{align}
Replacing Eq. \eqref{eq:x_attn_dash_j} in Eq. \eqref{eq:eviction_error}, we get
\begin{align}
    e_{\text{eviction},j} = & ||X_{\text{attn}} - X_{\text{attn},j}'||_{2} \nonumber
    \\
     =& \frac{\alpha_{j}}{1-\alpha_{j}} ||v_{j}-X_{\text{attn}}||_{2} = \frac{\alpha_{j}}{1-\alpha_{j}} ||VA^{T} - v_{j}||_{2} \nonumber
     \\
     = & c_{j}
    \label{eq:delta_X_attn_final}
\end{align}
Hence proved.
\end{proof}


 Using Eq. \eqref{eq:delta_X_attn_final} \textit{CAOTE} scores (or eviction error) for each token can be computed in parallel as the dependency on only on attention scores and value vectors. Note that this is the firsts formulation seamlessly integrating attention scores and value vectors into a single score. 
 Any norm can be used for computing \textit{CAOTE} score and based on empirical results we choose $L_2$-norm. Evicting multiple tokens using \textit{CAOTE} formulation is discussed in \cref{appendix:multi_token_caote}.



\subsection{\textit{CAOTE} with general score-based eviction methods}
\label{subsec:meta_caote}
The \textit{CAOTE} formulation allows the use of arbitrary scoring-based eviction methods to incorporate the values into their scoring mechanism, provided that the scores sum to 1.0. In practice, we can adjust the raw eviction scores without changing their relative order by simple normalizations (affine transformations). Let $H$ be the set of retention scores and $f^{\text{norm}}$ be the normalizing function. The \textit{CAOTE} score for general  eviction methods is given by:
\begin{align}
    c_{j} &=f^{\text{caote}}_{j}(f^{\text{norm}}(H),V)
    \\
    &= \frac{h_{j}^{\text{norm}}}{1-h_{j}^{\text{norm}}}||V(H^{\text{norm}})^{T}-v_{j}||_{2} 
    \label{eq:caote_score_general}
\end{align}
where, 
$h_{j}^{\text{norm}}=f^{\text{norm}}_{j}(H)$. We further discuss the generalization of \textit{CAOTE} to well-known token eviction methods in the following. 

\paragraph{\textit{CAOTE} for H2O}



We consider \textit{H2O} \cite{zhang2024h2o}, where the scores ($H = [h_{1}, \dots, h_{b+1}]$) are based on the sum of previous attention scores, leading to $\Sigma_{j=1}^{b+1} h_{j} > 1$ during generation-phase as proved in Theorem \ref{thm:h2o_scores} in \cref{subsec:h2o_scores}. 
In this case, simply dividing each token score by the sum of all scores maps the scores to the range $[0, 1]$ and ensures that new scores follow $\sum_{i=1}^{b+1} h_{i}^{\text{norm}} = 1$.
\begin{equation}
    h_{j}^{\text{norm}} = \frac{h_{j}}{\Sigma_{i=1}^{b+1} h_{i}}
    \label{eq:h2o_score_norm}
\end{equation}
For recent methods where all the scores are $\geq 0$, simply dividing by sum of all scores suffices. Note that for \textit{TOVA} \cite{oren2024transformers}, this summation is by default equal to $1$.

\subsection{\textit{FastCAOTE} Computation}
\label{subsec:fast_caote}
We also propose a compute-efficient version of \textit{CAOTE}, {\it FastCAOTE}, with negligible performance degradation and minimal compute overhead relative to \textit{CAOTE}. Here, the pre-eviction attention output ($X_{\text{attn}}$) is replaced with mean of values while everything else remains same, that is, \textit{CAOTE} score for token $j$ is:
\begin{align}
    c_{j} = \frac{\alpha_{j}}{1-\alpha_{j}}||\frac{1}{b+1}\Sigma_{i=1}^{b+1}v_{i} - v_{j}||_{2} 
\end{align}
We find empirically, that \textit{FastCAOTE} is highly correlated to \textit{CAOTE} by observing high ($\geq 0.8$) Spearman's correlation between them for each layer as shown in \cref{tab:correlation_caote_fast_caote}.

We provide theoretical inference time latency for using \textit{CAOTE} and \textit{FastCAOTE} in \cref{appendix:inference_latency}, where we observe that \textit{FastCAOTE} leads to minimal increase in latency. 

\section{Results}
\label{sec:results}
In this section, we demonstrate the efficacy of \textit{CAOTE} for boosting performance on state-of-the-art token eviction methods on a wide range of downstream benchmarks on recent frontier models: \texttt{Llama3} and \texttt{Qwen2.5}.
\subsection{Experiment Setup}

\paragraph{Tasks} We study the impact of \textit{CAOTE} on different token eviction methods by evaluating on LongBench \cite{bai-etal-2024-longbench}, covering single QA, multiple QA, single/multi-document summarization, synthetic, and code generation tasks.
We measure long-context perplexity on the Booksum dataset \cite{kryscinski2021booksum}, and lastly,  measure recall accuracy on Needle In A Haystack task \cite{liu2024lost,kamradt2023needle}. Details on context and generation lengths are in Appendix \cref{subsec:task_gen_len}

\paragraph{Baselines} We compare the performance of \textit{CAOTE} to various token eviction methods including:  \textit{H2O} \cite{zhang2024h2o}, \textit{TOVA} \cite{oren2024transformers}, and \textit{SnapKV} \cite{li2024snapkv}, on \texttt{Llama3} models: \texttt{Llama 3.2-3B-Instruct} and \texttt{Llama 3.1-8B-Instruct} \cite{dubey2024llama}, and \texttt{Qwen2.5} models: \texttt{Qwen 2.5-3B-Instruct} and \texttt{Qwen 2.5-7B-Instruct} \cite{yang2024qwen2} for all subsequent experiments. 


\paragraph{Budgets} We evaluated all methods with various KV cache budget sizes of $2048$, $4096$, $6144$, and $8192$, denoted by $2$k, $4$k, $6$k, and $8$k, respectively. 

\paragraph{Prompt consumption} Unlike other token eviction methods that assume to prefill prompt at once followed by KV cache eviction, we propose to consume tokens in block-wise manner as described in \cref{sec:background} with the block-size of $128$, i.e., at each inference of LLM during the prefill phase, there are $128$ new tokens incoming and being added to the cache, and $128$ tokens from the cache are evicted once the total number of tokens reaches the cache budget limit. 

\subsection{LongBench}
\begin{table*}[!t]
\caption{\textbf{LongBench results for Llama 3.1-8B and Llama 3.2-3B-Instruct. }
Higher number is better. 
We highlight the best performing methods within a given budget with \textbf{bold} and the second best with underline. 
} 
\label{table:longbench_reduced_llama_24}
\vspace{-5mm}
\begin{center}
\resizebox{\textwidth}{!}{
\begin{tabular}{clccccccccccccccccc}
\toprule

&
&
\multicolumn{3}{c}{Single Doc. QA} &
\multicolumn{3}{c}{Multi Doc. QA} &
\multicolumn{3}{c}{Summarization} &
\multicolumn{3}{c}{Few$\-$shot Learning} &
\multicolumn{2}{c}{Synthetic} &
\multicolumn{2}{c}{Code} &
\\

\cmidrule(lr){3-5}
\cmidrule(lr){6-8}
\cmidrule(lr){9-11}
\cmidrule(lr){12-14}
\cmidrule(lr){15-16}
\cmidrule(lr){17-18}

&
& 
{Narrative QA} & {Qasper} & {MF-en} & 
{HotpotQA} & {2WikiMQA} & {Musique} & 
{GovReport} & {QMSum} & {MultiNews} & 
{TREC} & {TriviaQA} & {SAMSum} & 
{PCount} & {PR-en} &
{Lcc} & {RB-P} & 
{Avg.} 
\\

\midrule

\multicolumn{2}{c}{Llama 3.1-8B} &
30.05
& 47.00 & 56.12 & 57.33 & 47.81 & 32.25 & 34.86 & 25.32 & 27.02 & 73.00 & 91.61 & 43.37 & 8.33 & 99.50 & 61.66 & 51.94 & 49.20 \\
\midrule

\multirow{10}{*}{2k} &
H2O & 
1.74 & 21.15 & 25.33 & 26.11 & 24.15 & 8.78 & 2.17 & 2.70 & 16.78 & 44.00 & 29.36 & 7.62 & 2.25 & 5.88 & 40.15 & 12.14 & 16.89 \\
\cmidrule(lr){2-19}

& 
\multicolumn{1}{r}{+ CAOTE} & 
14.32 & 38.34 & 45.97 & 37.77 & \textbf{42.51} & 22.06 & 29.57 & 15.11 & \textbf{27.02} &
62.00 & 63.60 & 27.34 & 2.00 & 15.50 & 56.99 & 32.87 & 33.31 \\

& 
\multicolumn{1}{r}{+ FastCAOTE} & 
15.15 & \textbf{41.27} & \textbf{46.6} & 39.91 & 40.02 & \textbf{24.55} & 30.05 & 16.19 & 26.95 & 63 & 62.39 & 26.86 & 3.08 & 17.5 & 56.87 & 34.75 & 34.07
\\

\cmidrule[0.8pt](lr){2-19}

&
TOVA & 
22.57 & 37.26 & 39.43 & 45.74 & 34.48 & 14.77 & 28.87 & 21.17 & 26.95 & 62.50 & 90.73 & 42.74 & 0.00 & 18.00 & \textbf{62.68} & 52.48 & 37.52 \\
\cmidrule(lr){2-19}

& 
\multicolumn{1}{r}{+ CAOTE} &
21.92 & 37.47 & 38.28 & 45.88 & 35.2 & 15 & 29.02 & 21.21 & 27.00 & 62.5 & 91.34 & 43.22 & 1.5 & 23 & 62.6 & \textbf{54.13} & 38.08 \\

& 
\multicolumn{1}{r}{+ FastCAOTE}& 
21.94 & 38.22 & 38.22 & \textbf{46.72} & 36.93 & 14.31 & 29.06 & 21.72 & 26.98 & 63 & \textbf{91.65} & \textbf{43.53} & 1.5 & 22 & 62.44 & 52.88 & 38.19 \\

\cmidrule[0.8pt](lr){2-19}

&
SnapKV & 
21.81 & 37.22 & 37.19 & 46.10 & 35.42 & 16.53 & 29.83 & 21.05 & 26.77 & 61.00 & 88.84 & 42.56 & \textbf{4.03} & 51.50 & 62.37 & 51.45 & 39.60 \\
\cmidrule(lr){2-19}

& 
\multicolumn{1}{r}{+ CAOTE} &
21.75 & 37.49 & 36.86 & 44.62 & 37.26 & 16.82 & \textbf{30.30} & 21.67 & 26.88 & 64 & 90.65 & 42.80 & 2.09 & 53.00 & 62.50 & 52.09 & \underline{40.05} \\

& 
\multicolumn{1}{r}{+ FastCAOTE} & 
\textbf{23.26} & 38.54 & 39.16 & 43.2 & 38.27 & 17.54 & 30.28 & \textbf{21.97} & 26.76 & \textbf{65.5} & 90.91 & 42.71 & 2.84 & \textbf{56.00} & 62.36 & 52.40 & \textbf{40.73} \\


\midrule

\multirow{10}{*}{4k} &
H2O & 
4.07 & 36.16 & 36.00 & 33.52 & 32.87 & 17.78 & 6.66 & 5.95 & 24.09 & 55.00 & 47.65 & 17.41 & 4.00 & 24.50 & 54.85 & 21.43 & 26.37 \\
\cmidrule(lr){2-19}

& 
\multicolumn{1}{r}{+ CAOTE} & 
20.28 & 46.08 & 51.45 & 47.38 & 46.05 & 30.89 & 33.39 & 20.8 & 26.93 & 69 & 80.12 & 38.27 & 4.31 & 32 & 59.22 & 40.51 & 40.42 \\

& 
\multicolumn{1}{r}{+ FastCAOTE} & 
24.4 & 44.32 & 48.11 & 48.19 & 43.69 & 21.12 & 31.55 & 22.36 & 26.98 & 65 & 91.18 & 43.11 & 2 & 46.5 & 61.62 & 53.35 & 42.09 \\

\cmidrule[0.8pt](lr){2-19}

&
TOVA & 
22.68 & 44.55 & 47.87 & 46.76 & 44.54 & 20.56 & 30.95 & 22.13 & 26.96 & 61.50 & 90.56 & 43.27 & 3.00 & 43.50 & 61.62 & 53.40 & 41.49 \\
\cmidrule(lr){2-19}

& 
\multicolumn{1}{r}{+ CAOTE} &
24.68 & 43.88 & 48.07 & 49.64 & 44.91 & 22.57 & 31.25 & 22.25 & 26.98 & 63 & 91.29 & 43.29 & 2.5 & 46.5 & 61.6 & 53.45 & 42.24 \\


& 
\multicolumn{1}{r}{+ FastCAOTE}& 
24.4 & 44.32 & 48.11 & 48.19 & 43.69 & 21.12 & 31.55 & 22.36 & 26.98 & 65 & 91.18 & 43.11 & 2 & 46.5 & 61.62 & 53.35 & 42.09 \\

\cmidrule[0.8pt](lr){2-19}

&
SnapKV & 
24.79 & 44.22 & 47.30 & 48.49 & 46.73 & 20.55 & 32.19 & 22.68 & 26.95 & 67.50 & 90.98 & 43.14 & 5.17 & 89.50 & 61.44 & 51.20 & 45.18 \\
\cmidrule(lr){2-19}

& 
\multicolumn{1}{r}{+ CAOTE} &
24.41 & 43.16 & 47.77 & 50.87 & 44.11 & 21.04 & 32.51 & 22.98 & 26.93 & 69 & 91.31 & 43.18 & 3.33 & 92 & 61.04 & 51.74 & \underline{45.34} \\

& 
\multicolumn{1}{r}{+ FastCAOTE} & 
24.12 & 44.59 & 47.39 & 50.82 & 44.07 & 22.38 & 32.33 & 22.92 & 27.01 & 69 & 91.31 & 43.53 & 4.58 & 94.5 & 61.31 & 52.11 & \textbf{45.75} \\

\midrule
\midrule
\addlinespace
\multicolumn{2}{c}{Llama 3.2-3B} &
23.76 & 40.23 & 50.09 & 50.69 & 42.29 & 26.84 & 33.09 & 24.30 & 25.21 & 72.50 & 90.11 & 42.58 & 3.00 & 96.50 & 56.22 & 56.52 & 45.87 \\
\midrule
\multirow{10}{*}{2k} & H2O & 1.63 & 19.96 & 20.20 & 18.02 & 19.56 & 2.88 & 0.78 & 1.55 & 15.97 & 41.00 & 21.97 & 9.83 & 0.50 & 0.50 & 39.71 & 13.91 & 14.25
\\
\cmidrule[0.8pt](lr){2-19}
& 
\multicolumn{1}{r}{+ CAOTE} & 6.38 & \textbf{34.36} & \textbf{40.6}  & 32.52 & 31.08 & 12.69  & 27.36 & 15.04 & 24.6  & 59 & 52.83 & 26.78 & 3.7 & 7.56 & 51.09 & 36.33 & 28.87
\\
& 
\multicolumn{1}{r}{+ FastCAOTE} & 7.27 & 34.23 & 39.74  & 32.22 & 30.08 & 12.63  & \textbf{27.86} & 15.48 & 25.15 & 60.5 & 53.09 & 26.94 & 2.17 & 8.12 & 51.2 & 35.06 & 28.86
\\
\cmidrule[0.8pt](lr){2-19}
& TOVA & 17.14 & 30.14 & 32.44 & 35.96 & 30.05 & 13.08 & 26.15 & 19.70 & 25.04 & 56.50 & 87.81 & 40.48 & 2.50 & 11.50 & 55.51 & 52.36 & 33.52 
\\
\cmidrule[0.8pt](lr){2-19}
&
\multicolumn{1}{r}{+ CAOTE} & 17.75 & 30.45 & 32.19 & 37.53 & 29.35 & 13.33 & 26.92 & 19.93 & 25.18 & 60.5 & 88.08 & \textbf{41.65} & 1.00 & 12.5 & 54.92 & 53.22 & 34.03
\\
& 
\multicolumn{1}{r}{+ FastCAOTE} & 17.93 & 30.52 & 33.1 & 37.01 & 30.7 & \textbf{13.88} & 26.39 & 20.28 & 24.96 & 60.5 & \textbf{88.95} & 41.27 & 2.00 & 12.5 & 55.65 & 53.56 & 34.33
\\
\cmidrule[0.8pt](lr){2-19}
& SnapKV & 17.38 & 31.37 & 31.48 & 37.77 & 30.05 & 11.54 & 27.03 & 19.93 & 24.97 & 59.00 & 88.13 & 40.48 & 3.50 & 32.50 & 56.32 & 55.91 & 35.46 
\\
\cmidrule[0.8pt](lr){2-19}
& \multicolumn{1}{r}{+CAOTE} & \textbf{19.11} & 33.12 & 31.09 & 38.68 & \textbf{32.09} & 12.31 & 27.48 & \textbf{20.38} & 25.20 & 64 & 87.7 & 40.78 & 2.5 & 35 & \textbf{57.03} & 56.21 & \underline{36.42}
\\
& \multicolumn{1}{r}{+FastCAOTE} & 18.96 & 32.97 & 33.61  & \textbf{39.00} & 31.36 & 12.35 & 27.48 & 20.15 & \textbf{25.24} & \textbf{65} & 87.2 & 40.7 & \textbf{4.5} & \textbf{36.5} & 56.06 & \textbf{57.12} & \textbf{36.76}
\\
\midrule
\multirow{10}{*}{4k} & H2O & 2.92 & 31.94 & 33.23 & 24.49 & 28.08 & 7.55 & 5.44 & 6.30 & 22.77 & 53.00 & 38.85 & 20.33 & 1.50 & 7.50 & 51.23 & 22.94 & 22.38
\\
\cmidrule[0.8pt](lr){2-19}
& \multicolumn{1}{r}{+CAOTE} & 12.87 & 40.79 & 47.56 & 40.28 & 39.07 & 16.61 & 30.82 & 19.65 & 25.12  & 65.5 & 69.29 & 34.16 & 2.35 & 17 & 55.32 & 45.12 & 35.09
\\
& \multicolumn{1}{r}{+FastCAOTE} & 11.85 & 40.41 & 47.93  & 40.81 & 38.93 & 17.36  & 31.22 & 19.67 & 25.1 & 65 & 71.25 & 34.89 & 3.5 & 15 & 55.5 & 44.3 & 35.17
\\
\cmidrule[0.8pt](lr){2-19}
& TOVA & 20.52 & 39.53 & 42.47 & 44.12 & 38.42 & 18.22 & 29.36 & 21.36 & 24.96 & 63.50 & 88.98 & 41.50 & 3.00 & 23.50 & 55.72 & 56.66 & 38.24
\\
\cmidrule[0.8pt](lr){2-19}
& \multicolumn{1}{r}{+CAOTE} &
19.59 & 39.79 & 42.03 & 45.25 & 37.07 & 19.3 & 29.39 & 21.57 & 24.92 & 63 & 89.14 & 41.77 & 3.00 & 29.5 & 55.68 & 56.19 & 38.57
\\
& \multicolumn{1}{r}{+FastCAOTE} &
19.77 & 39.23 & 43.13  & 45.28 & 37.04 & 18.82 & 29.25 & 21.94 & 24.96 & 63 & 88.64 & 41.92 & 3.5 & 29 & 55.68 & 56.41 & 38.6
\\
\cmidrule[0.8pt](lr){2-19}
& SnapKV & 19.85 & 39.22 & 39.86 & 46.70 & 37.98 & 16.64 & 29.79 & 21.21 & 25.01 & 65.50 & 89.35 & 40.95 & 2.50 & 62.50 & 55.74 & 56.88 & 40.60 
\\
\cmidrule[0.8pt](lr){2-19}
& \multicolumn{1}{r}{+CAOTE} &
20.1 & 39.74 & 41.01 & 45.64 & 38.26 & 18.64 & 30.07 & 21.53 & 24.98 & 67.5 & 89.08 & 41.78 & 3.00 & 67  & 55.73 & 56.71 & \underline{41.30} 
\\
& \multicolumn{1}{r}{+FastCAOTE} &
19.68 & 39.24 & 41.03  & 44.52 & 39.09 & 18.62 & 30.15 & 21.72 & 24.97 & 67 & 88.86 & 41.24 & 3.00 & 71 & 55.67 & 56.64 & \textbf{41.40}
\\
\bottomrule
\end{tabular}
}
\end{center}
\end{table*}
\begin{table*}[!t]
\caption{\textbf{LongBench results for Qwen 2.5-7B/2.5-3B-Instruct.}
Higher number is better. 
We highlight the best performing methods within a given budget with \textbf{bold} and the second best with underline. 
} 
\label{table:longbench_reduced_qwen_24}
\vspace{-5mm}
\begin{center}
\resizebox{\textwidth}{!}{
\begin{tabular}{clccccccccccccccccc}
\toprule

&
&
\multicolumn{3}{c}{Single Doc. QA} &
\multicolumn{3}{c}{Multi Doc. QA} &
\multicolumn{3}{c}{Summarization} &
\multicolumn{3}{c}{Few$\-$shot Learning} &
\multicolumn{2}{c}{Synthetic} &
\multicolumn{2}{c}{Code} &
\\

\cmidrule(lr){3-5}
\cmidrule(lr){6-8}
\cmidrule(lr){9-11}
\cmidrule(lr){12-14}
\cmidrule(lr){15-16}
\cmidrule(lr){17-18}

&
& 
{Narrative QA} & {Qasper} & {MF-en} & 
{HotpotQA} & {2WikiMQA} & {Musique} & 
{GovReport} & {QMSum} & {MultiNews} & 
{TREC} & {TriviaQA} & {SAMSum} & 
{PCount} & {PR-en} &
{Lcc} & {RB-P} & 
{Avg.} 
\\

\midrule

\multicolumn{2}{c}{Qwen 2.5-7B} &
 15.75 & 16.94 & 32.38 & 11.89 & 11.88 & 7.95 & 34.33 & 19.91 & 22.67 & 65.5 & 87.05 & 44.75 & 4.22 & 93.08 & 57.74 & 61.84 & 36.74
 \\
\midrule

\multirow{10}{*}{2k} &
H2O & 
2.39 & 7.29 & 12.42 & 8.55 & 11.06 & 2.73 & 3.62 & 6.6 & 15.69 & 42.5 & 28.21 & 10.63 & 0.65 & 0 & 35.1 & 18.77 & 12.89
\\
\cmidrule(lr){2-19}
& 
\multicolumn{1}{r}{+ CAOTE} & 
4.55 & 14.3 & 27.58 & 11.33 & 13.55 & 7.76 & 26.65 & 15.62 & 22.93 & 57 & 49.78 & 27.74 & 1.54 & 11.08 & 51.45 & 32.7 & 23.47 
\\
& 
\multicolumn{1}{r}{+ FastCAOTE} & 
4.8 & 12.79 & 28.72 & 12.94 & 13.25 & 7.53 & 27.06 & 14.46 & 22.84 & 59 & 48.23 & 26.4 & 2.53 & 11.54 & 52.85 & 32.93 & 23.62
\\

\cmidrule[0.8pt](lr){2-19}

&
TOVA & 
8.49 & 14.01 & 21.04 & 14 & 11.51 & 5.09 & 27.43 & 17.84 & 22.83 & 56.5 & 79.56 & 40.55 & 2.43 & 9.29 & 55.99 & 56.15 & 27.67 
\\
\cmidrule(lr){2-19}

& 
\multicolumn{1}{r}{+ CAOTE} &
10.46 & 14.82 & 25.06 & 14.62 & 11.73 & 6.01 & 27.66 & 18.02 & 22.78 & 57.5 & 79.39 & 40.87 & 2.5 & 11.25 & 56.22 & 56.51 & 28.46 
\\

& 
\multicolumn{1}{r}{+ FastCAOTE}& 
10.08 & 13.58 & 25.28 & 14.44 & 12.14 & 5.24 & 27.34 & 18.31 & 23.11 & 55.5 & 78.51 & 41.67 & 2.7 & 10.54 & 56.56 & 58.05 & 28.32
\\

\cmidrule[0.8pt](lr){2-19}

&
SnapKV & 
11.6 & 12.45 & 23.66 & 12.38 & 10.64 & 7.03 & 27.57 & 18.27 & 22.85 & 58 & 81.78 & 41.13 & 3.76 & 19.42 & 55.83 & 56.53 & 28.93
\\
\cmidrule(lr){2-19}

& 
\multicolumn{1}{r}{+ CAOTE} &
 14.02 & 12.23 & 24.55 & 16.45 & 10.35 & 8.59 & 27.77 & 18.91 & 22.87 & 56 & 80.58 & 40.43 & 2.38 & 21.52 & 55.17 & 56.03 & \underline{29.24}
 \\

& 
\multicolumn{1}{r}{+ FastCAOTE} & 
 14.26 & 14.11 & 24.11 & 15.31 & 11.35 & 7.88 & 27.95 & 18.86 & 22.74 & 56.5 & 80.92 & 41.49 & 3.8 & 22.42 & 55.89 & 57.43 & \textbf{29.69}
 \\


\midrule

\multirow{10}{*}{4k} &
H2O & 
1.99 & 11.92 & 19.88 & 10.24 & 10.12 & 4.73 & 9.08 & 10.14 & 20.85 & 51.00 & 37.37 & 20.57 & 3.16 & 6.43 & 52.14 & 29.09 & 18.67
\\
\cmidrule(lr){2-19}

& 
\multicolumn{1}{r}{+ CAOTE} & 
4.78 & 18.06 & 32.49 & 16.23 & 17.28 & 9.57 & 29.81 & 18.04 & 22.86 & 59.5 & 63.05 & 36.91 & 2.7 & 28.25 & 55.13 & 42.42 & 28.57
\\

& 
\multicolumn{1}{r}{+ FastCAOTE} & 
5.69 & 16.99 & 32.62 & 18.22 & 16.58 & 10.48 & 30.3 & 17.71 & 22.88 & 59.5 & 62.95 & 36.29 & 2.1 & 27.65 & 56.3 & 40.65 & 28.56
\\

\cmidrule[0.8pt](lr){2-19}

&
TOVA & 
12.83 & 17.03 & 27.01 & 16.8 & 13.37 & 8.05 & 29.21 & 19.05 & 22.73 & 58.5 & 82.67 & 42.71 & 1.67 & 15 & 56.69 & 56.59 & 29.99
\\
\cmidrule(lr){2-19}

& 
\multicolumn{1}{r}{+ CAOTE} &
12.97 & 14.99 & 27.53 & 17.94 & 12.93 & 9.21 & 29.76 & 19.7 & 22.92 & 58 & 82.03 & 43.14 & 2.15 & 17.25 & 57.32 & 59.37 & 30.98
\\ 

& 
\multicolumn{1}{r}{+ FastCAOTE}& 
14.52 & 16.71 & 26.97 & 18.73 & 13.84 & 9.59 & 29.47 & 19.45 & 22.87 & 59.5 & 82.96 & 42.42 & 2.6 & 20.33 & 57.22 & 58.42 & 30.98
\\

\cmidrule[0.8pt](lr){2-19}

&
SnapKV & 
14.35 & 13.45 & 28.28 & 16.33 & 11.74 & 8.12 & 29.71 & 19.18 & 22.82  & 57 & 83.8 & 43.27 & 2.41 & 39.83 & 58.12 & 58.67 & 31.69
\\
\cmidrule(lr){2-19}

& 
\multicolumn{1}{r}{+ CAOTE} &
15.07 & 14.34 & 28.7 & 16.7 & 12.89 & 10.54 & 30.03 & 19.58 & 22.73 & 59.5 & 83.12 & 42.56 & 3.17 & 55.92 & 57.34 & 58.85 & \textbf{33.19}
\\

& 
\multicolumn{1}{r}{+ FastCAOTE} & 
17.12 & 14.69 & 27.6 & 17.52 & 13.69 & 9.96 & 30.24 & 20.02 & 22.88 & 58.5 & 81.13 & 42.31 & 4.06 & 53.33 & 57.51 & 58.77 & \underline{33.08}
\\

\midrule
\midrule
\addlinespace
\multicolumn{2}{c}{Qwen 2.5-3B} &
18.08 & 22.49 & 39.72 & 27.86 & 20.45 & 18.93 & 32.8 & 23.74 & 24.89 & 67.5 & 85.05 & 43.88 & 5 & 40.97 & 51.91 & 47.53 & 35.68
\\
\midrule
\multirow{10}{*}{2k} & H2O & 1.8 & 9.18 & 11.62 & 8.54 & 7.31 & 2.77 & 5.93 & 6.99 & 16.89 & 38 & 21.87 & 7.69 & 1 & 3 & 37.36 & 22.9 & 12.68
\\
\cmidrule[0.8pt](lr){2-19}
& 
\multicolumn{1}{r}{+ CAOTE} & 6.9 & 22.71 & 28.09 & 15.23 & 18.19 & 4.95 & 29.53 & 17.68 & 24.74 & 52.5 & 45.81 & 26.95 & 1.92 & 6.16 & 45.81 & 36.48 & 23.98 
\\
& 
\multicolumn{1}{r}{+ FastCAOTE} & 7.03 & 22.37 & 28.88 & 15.34 & 16.95 & 5.19 & 29.13 & 18.06 & 25.03 & 54.5 & 46.35 & 25.35 & 2.23 & 7.22 & 45.7 & 36.59 & 24.12
\\
\cmidrule[0.8pt](lr){2-19}
& TOVA & 11.69 & 14.94 & 25.33 & 17.29 & 12.58 & 5.91 & 26.67 & 21.49 & 24.78 & 51.5 & 68.8 & 41.79 & 0.23 & 6 & 49.79 & 48.6 & 26.71
\\
\cmidrule[0.8pt](lr){2-19}
&
\multicolumn{1}{r}{+ CAOTE} & 11.17 & 15.23 & 27.42 & 18.94 & 13.1 & 6.94 & 27.01 & 21.62 & 24.86 & 57.5 & 68.38 & 42.11 & 0.82 & 4.88 & 49.36 & 48.13 & 27.34
\\
& 
\multicolumn{1}{r}{+ FastCAOTE} & 11.04 & 15.36 & 27.72 & 19.8 & 13.65 & 6.37 & 27.17 & 22.08 & 24.64 & 57 & 69.13 & 42.48 & 0.77 & 5.25 & 48.36 & 48.58 & 27.46
\\
\cmidrule[0.8pt](lr){2-19}
& SnapKV & 11.7 & 13.91 & 24.28 & 14.8 & 10.89 & 7.42 & 27.4 & 21.63 & 24.64 & 54.5 & 75.35 & 42.72 & 2.5 & 18.33 & 49.65 & 50.59 & 28.14
\\
\cmidrule[0.8pt](lr){2-19}
& \multicolumn{1}{r}{+CAOTE} & 12.69 & 14.88 & 26.16 & 13.93 & 12.21 & 7.07 & 27.48 & 20.99 & 24.75 & 61 & 75.58 & 42.08 & 4 & 21.29 & 49.94 & 52.38 & \textbf{29.15}
\\
& \multicolumn{1}{r}{+FastCAOTE} & 12.03 & 14.56 & 24.82 & 14.66 & 10.83 & 7.89 & 27.51 & 20.98 & 24.67 & 62.5 & 75.51 & 41.53 & 2 & 17 & 49.02 & 50.83 & \underline{28.52}
\\
\midrule
\multirow{10}{*}{4k} & H2O & 
2.82 & 17.34 & 23.27 & 10.18 & 10.47 & 3.03 & 11.06 & 10.73 & 22.93 & 50.75 & 34.93 & 18.03 & 4.35 & 7.32 & 47.74 & 29.42 & 19.02
\\
\cmidrule[0.8pt](lr){2-19}
& \multicolumn{1}{r}{+CAOTE} & 
7.63 & 24.16 & 35.29 & 20.17 & 17.67 & 12.61 & 31.14 & 19.04 & 25.01 & 62.5 & 64.84 & 34.19 & 4.25 & 18.37 & 49.79 & 41.45 & 29.26
\\
& \multicolumn{1}{r}{+FastCAOTE} & 
8.58 & 23.45 & 33.14 & 21.72 & 16.11 & 12.26 & 31.11 & 19.76 & 25.04 & 62 & 65.01 & 35.15 & 4.6 & 17.88 & 50.05 & 40.03 & 29.12
\\
\cmidrule[0.8pt](lr){2-19}
& TOVA & 
12.19 & 18.31 & 32.56 & 20.58 & 13.8 & 7.74 & 28.82 & 22.27 & 24.98 & 59 & 80.66 & 43.05 & 1.11 & 9.56 & 49.93 & 46.74 & 29.46
\\
\cmidrule[0.8pt](lr){2-19}
& \multicolumn{1}{r}{+CAOTE} &
13.16 & 18.67 & 30.74 & 19.33 & 15.7 & 7.32 & 28.93 & 22.14 & 24.91 & 59.5 & 78.54 & 43.57 & 1.55 & 8.25 & 49.4 & 47.31 & 29.31
\\
& \multicolumn{1}{r}{+FastCAOTE} &
12.2 & 18.55 & 32.29 & 19.3 & 15.13 & 7.23 & 29.12 & 22.44 & 24.97 & 60 & 78.8 & 43.12 & 1.5 & 10.25 & 49.6 & 47.74 & 29.52
\\
\cmidrule[0.8pt](lr){2-19}
& SnapKV & 
12.98 & 2.21 & 31.77 & 18.33 & 14.41 & 10.83 & 29.14 & 22.38 & 24.89 & 61 & 84.17 & 42.63 & 3.75 & 25.42 & 50.22 & 48.77 & 30.18
\\
\cmidrule[0.8pt](lr){2-19}
& \multicolumn{1}{r}{+CAOTE} &
13.65 & 20.35 & 32.62 & 19.36 & 15.27 & 11.42 & 29.47 & 22.44 & 24.78 & 64 & 82.6 & 43 & 4.25 & 24.46 & 50.37 & 49.28 & \textbf{31.71} 
\\
& \multicolumn{1}{r}{+FastCAOTE} &
13.46 & 19.92 & 32.53 & 20.44 & 13.64 & 8.44 & 29.56 & 22.14 & 24.93 & 64.5 & 82.73 & 43.26 & 2 & 24.58 & 50.14 & 50.08 & \underline{31.40}
\\
\bottomrule
\end{tabular}
}
\end{center}
\end{table*}
We present the accuracy of \texttt{Llama 3.1-8B-Instruct}, \texttt{Llama 3.2-3B-Instruct} and, \texttt{Qwen 2.5-3B-Instruct}, \texttt{Qwen 2.5-7B-Instruct} using baseline eviction methods with budget of $2$k, $4$k, both with and without \textit{CAOTE} in \cref{table:longbench_reduced_llama_24} and \cref{table:longbench_reduced_qwen_24}. We observe that the best average performance is given by \textit{SnapKV-FastCAOTE} for the \texttt{Llama3} models, while for \texttt{Qwen2.5} models \textit{SnapKV-CAOTE} performs the best. \textit{H2O} shows $>30\%$ improvement with \textit{CAOTE}, while \textit{TOVA, SnapKV} also show overall improvements, making their average accuracy closer to dense accuracy. Moreover, we obserive that for some \textbf{QA tasks (Qasper, MF-en, Musique, 2WikiMQA), \textit{H2O-(Fast)CAOTE} performs best}. We have bolded the best accuracy eviction method for $2$k budget in \cref{table:longbench_reduced_llama_24}.

Additional results for the $6$k, $8$k budget are shown in \cref{table:longbench_reduced_llama_68}, \cref{table:longbench_reduced_qwen_68} for \texttt{Llama3} and \texttt{Qwen2.5} respectively, in \cref{subsec:appendix_results_longbench}, which follow a trend similar to the $2$k, $4$k budgets.

\subsection{Perplexity}
\begin{table*}[t]
\vspace{-2mm}
\caption{\textbf{Perplexity difference between different eviction methods with dense baseline.} The lower is better. Negative entry in table means the method performs better than dense baseline. The PPL of Llama 3.2-3B-Instruct and Llama 3.1-8B-Instruct is 15.4911 and 9.833 respectively.}
\vspace{-2mm}
\label{table:PPL_wide}
\begin{center}\resizebox{0.8\textwidth}{!}{%
\begin{tabular}{cccccccccc}

\toprule

\multirow{2}{*}{Budget} & \multicolumn{3}{c}{H2O} & \multicolumn{3}{c}{TOVA} & 
\multicolumn{3}{c}{SnapKV}
\\

\cmidrule[0.5pt](lr){2-4}
\cmidrule[0.5pt](lr){5-7}
\cmidrule[0.5pt](lr){8-10}

& & +CAOTE & +FastCAOTE & & +CAOTE & +FastCAOTE & & +CAOTE & +FastCAOTE
\\

\midrule[0.7pt]
\multicolumn{10}{c}{Llama 3.1-8B-Instruct}
\\
\midrule

2k & 2.007 & 1.884 & 1.891 & -0.046 & -0.088 & -0.085 & -0.019 & -0.097 & \textbf{-0.098}
\\ 
4k & 1.284 & 1.079 & 1.061 & -0.047 & -0.060 & -0.058 & -0.0483 & \textbf{-0.080} & -0.079
\\
6k & 0.843 & 0.716 & 0.703 & -0.035 & -0.0366 & \textbf{-0.085} & -0.036 & -0.043 & -0.045
\\

\midrule[0.7pt]
\multicolumn{10}{c}{Llama 3.2-3B-Instruct}
\\
\midrule

2k & 3.814 & 3.561 & 3.563 & 0.493 & 0.442 & \textbf{0.432} & 0.555 & 0.451 & 0.435
\\ 
4k & 2.460 & 2.142 & 2.128 & 0.175 & 0.150 & \textbf{0.144} & 0.223 & 0.152 & 0.144
\\
6k & 1.369 & 1.219 & 1.187 & 0.065 & 0.057 & 0.057 & 0.076 & 0.056 & \textbf{0.044}
\\
8k & 0.589 & 0.462 & 0.448 & 0.023 & 0.012 & 0.011 & 0.020 & 0.007 & \textbf{0.005}
\\
\bottomrule
\end{tabular}
}%
\vspace{-5mm}
\end{center}
\end{table*}
We use the Booksum dataset \cite{kryscinski2021booksum} to measure generation perplexity of different eviction methods for various budgets. 
In \cref{table:PPL_wide}, we show perplexity gap between a model using a given eviction strategy and that of the model without token eviction with cache budgets of $2$k, $4$k and $6$k. 
We observe that when \textit{CAOTE} is applied to existing eviction methods, the perplexity either improves or surpasses the perplexity of the baseline model. 
\textit{TOVA-FastCAOTE}, \textit{SnapKV-CAOTE}, and \textit{SnapKV-FastCAOTE} perform best for $6$k, $4$k, $2$k budgets, respectively, for \texttt{Llama 3.1-8B-Instruct}; for \texttt{Llama 3.2-3B-Instruct}, \textit{TOVA-FastCAOTE} performs best with $2$k and $4$k budgets and \textit{SnapKV-FastCAOTE} beats other methods using $6$k and $8$k. Perplexity results for \texttt{Qwen2.5} models are shown in \cref{table:PPL_qwen_wide} in \cref{subsec:appendix_results_ppl}.

\subsection{Needle In A HayStack}
\vspace{1mm}
\begin{figure*}[ht!]
     \centering
     \begin{subfigure}[b]{0.33\linewidth}
         \centering
         \includegraphics[width=\linewidth]{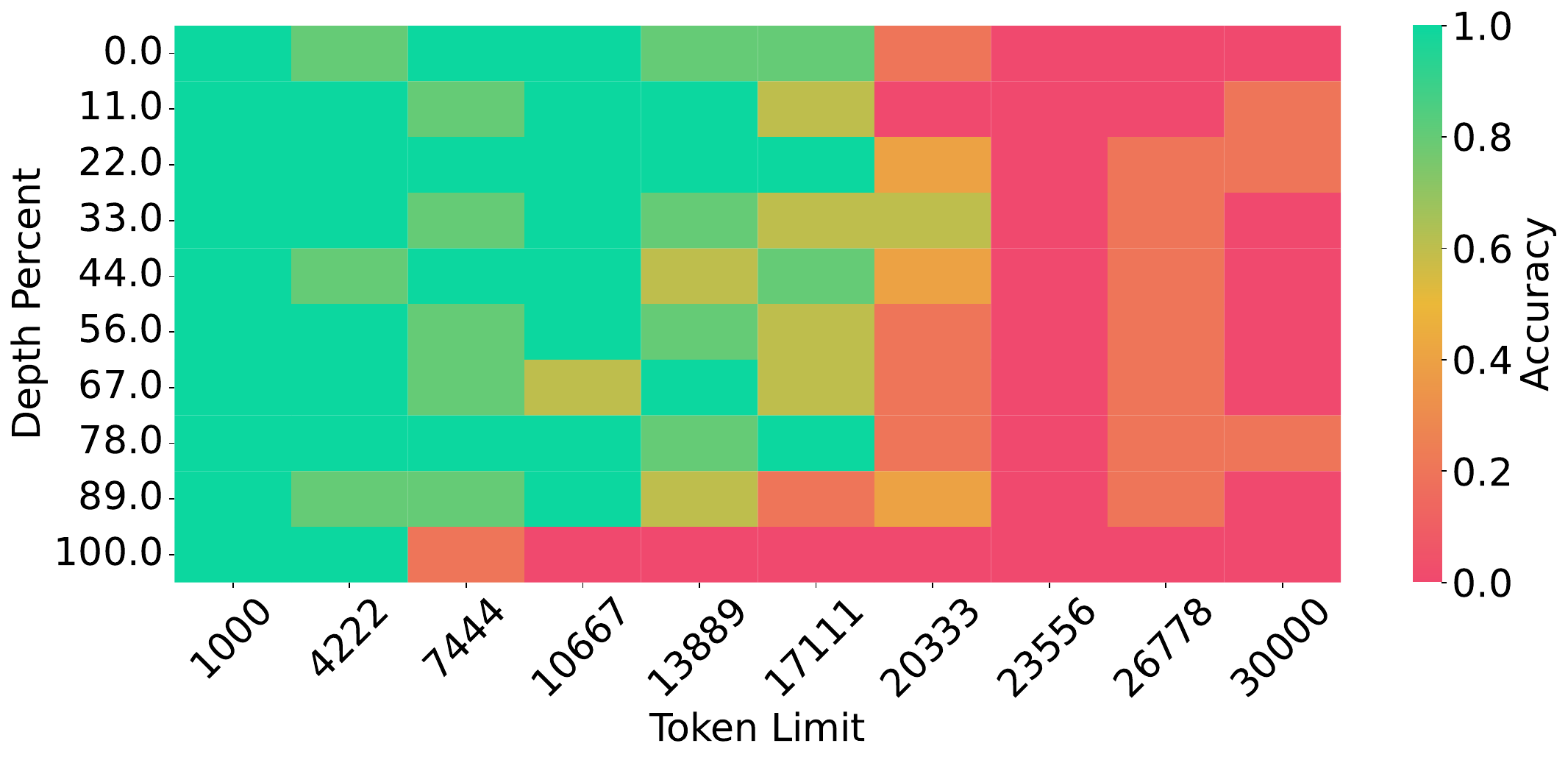}
         \caption{H2O}
     \end{subfigure}
     \hfill
     \begin{subfigure}[b]{0.33\linewidth}
         \centering
         \includegraphics[width=\linewidth]{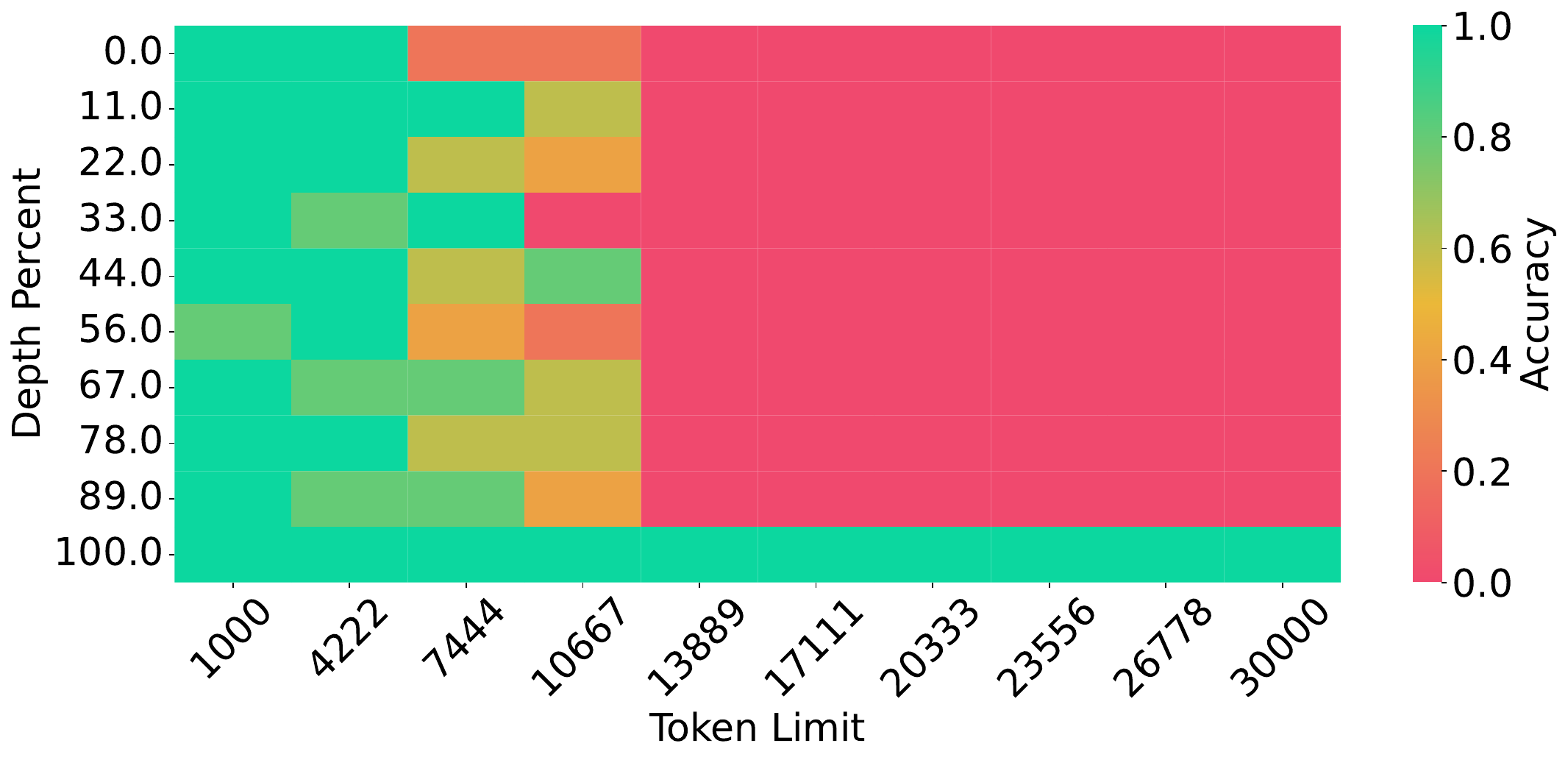}
         \caption{TOVA}
     \end{subfigure}
     \begin{subfigure}[b]{0.33\linewidth}
         \centering
         \includegraphics[width=\linewidth]{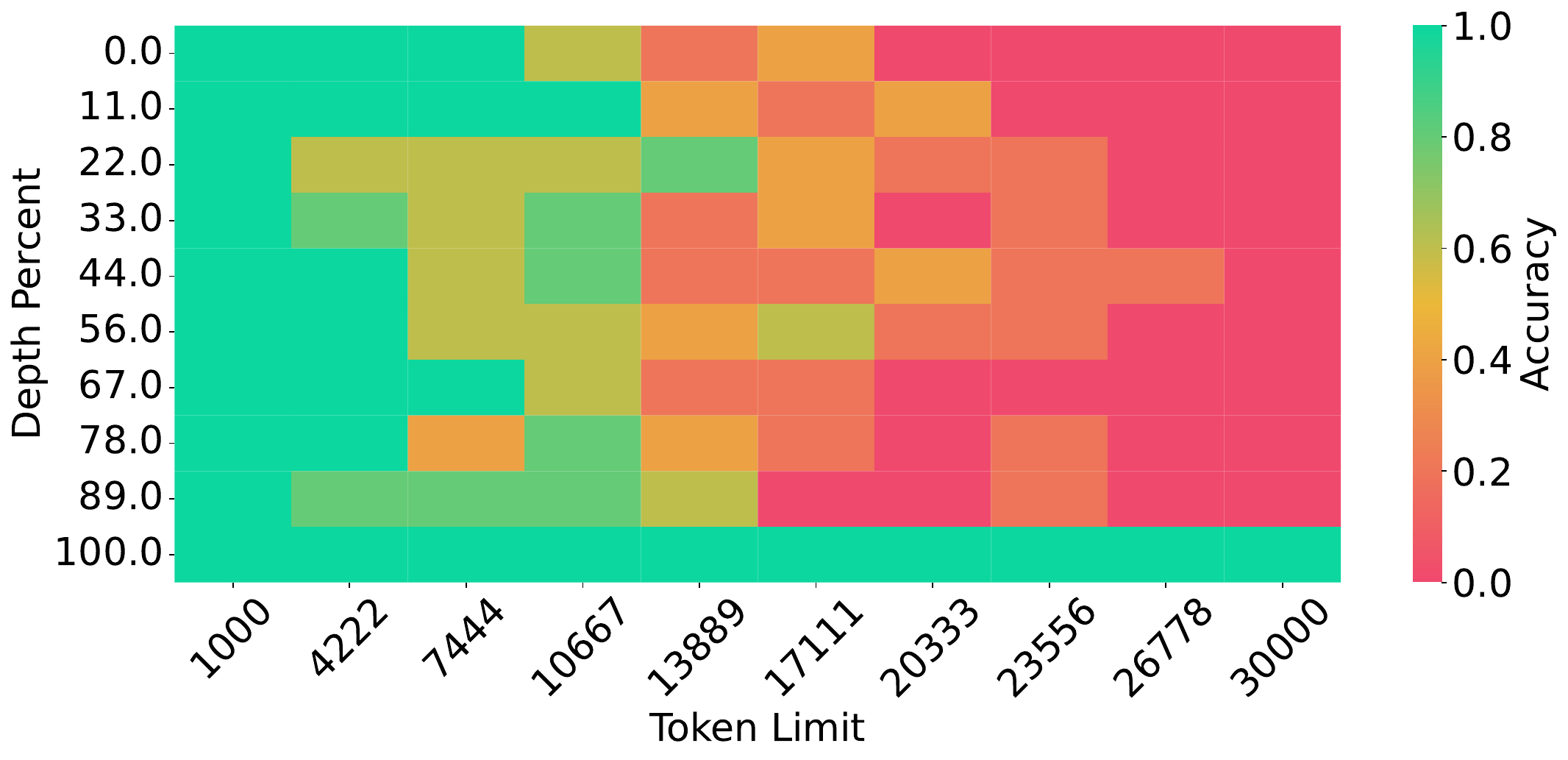}
         \caption{SnapKV}
    \end{subfigure}
    \\ 
     \begin{subfigure}[b]{0.33\linewidth}
         \centering
         \includegraphics[width=\linewidth]{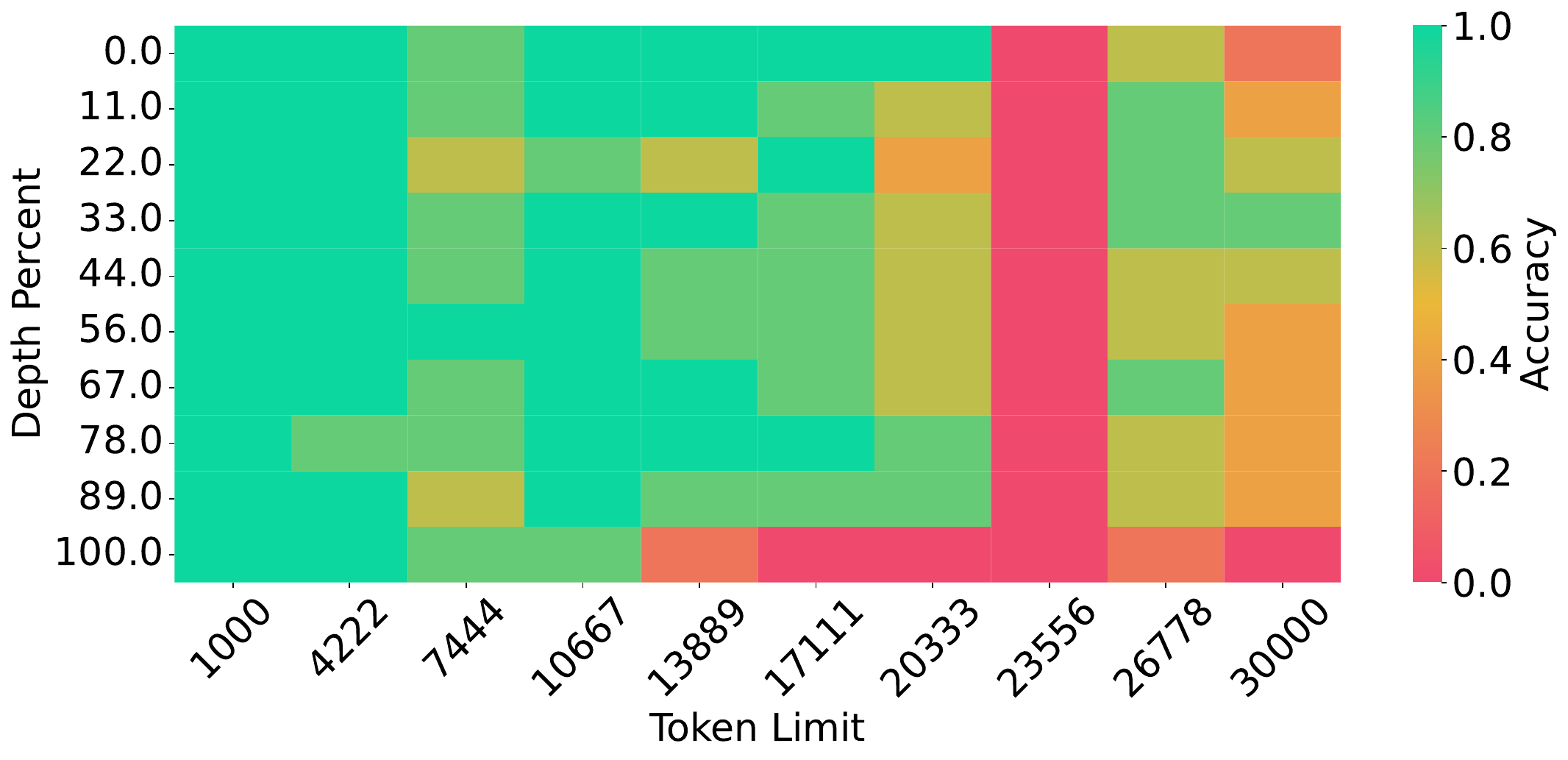}
         \caption{H2O-CAOTE}
     \end{subfigure}
     \hfill
     \begin{subfigure}[b]{0.33\linewidth}
         \centering
         \includegraphics[width=\linewidth]{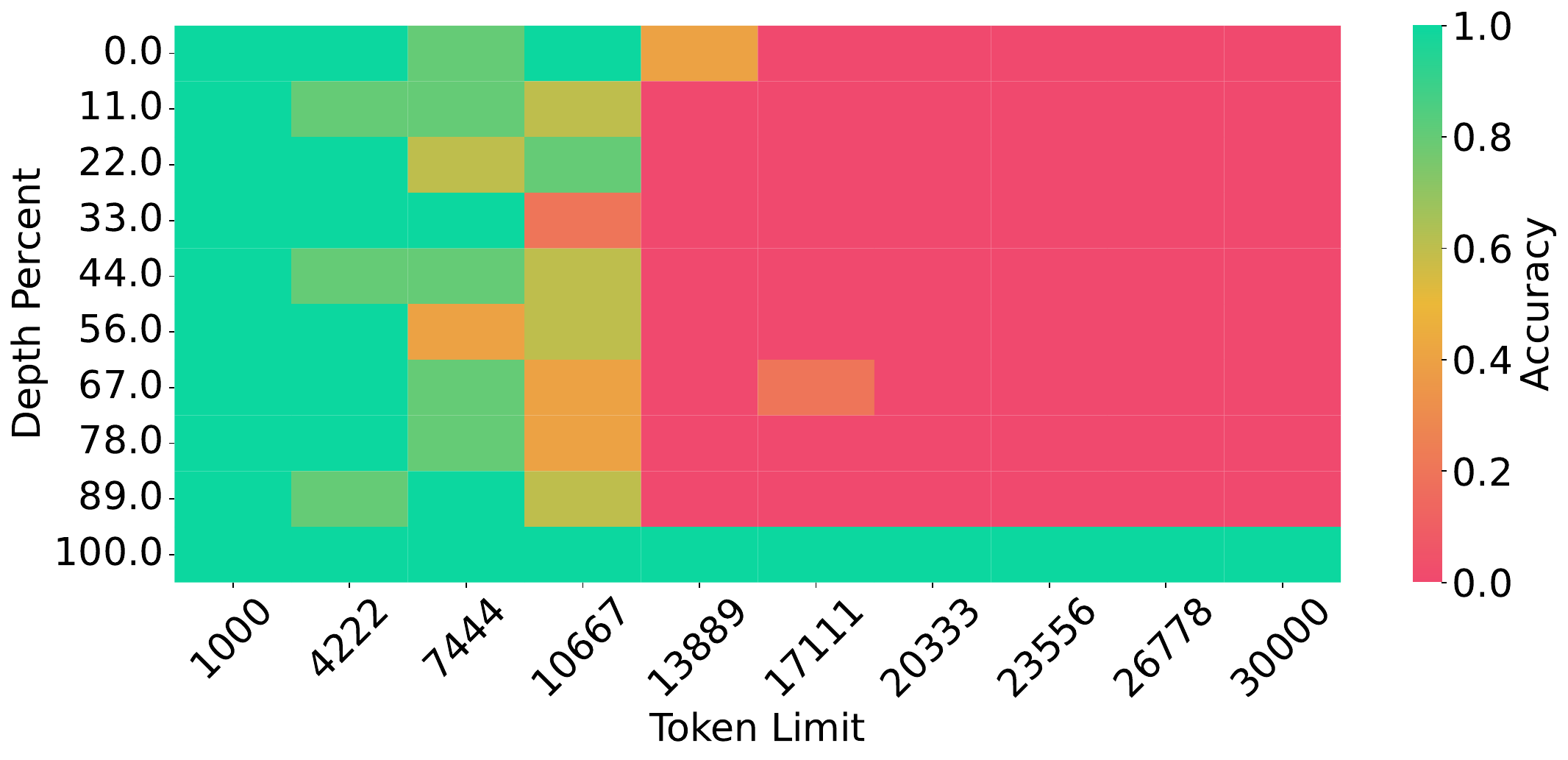}
         \caption{TOVA-CAOTE}
     \end{subfigure}
     \begin{subfigure}[b]{0.33\linewidth}
         \centering
         \includegraphics[width=\linewidth]{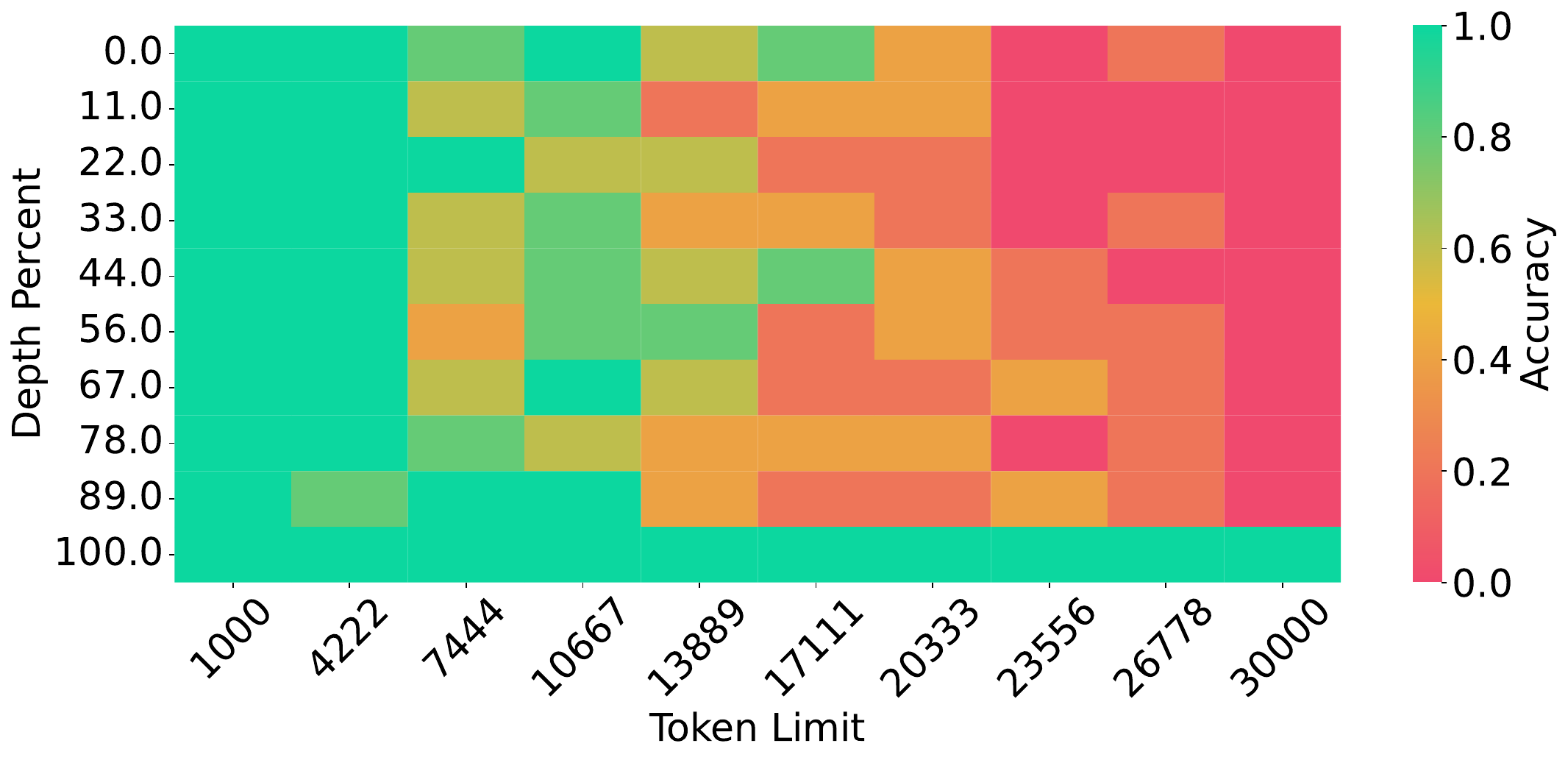}
         \caption{SnapKV-CAOTE}
    \end{subfigure}
    \\ 
     \begin{subfigure}[b]{0.33\linewidth}
         \centering
         \includegraphics[width=\linewidth]{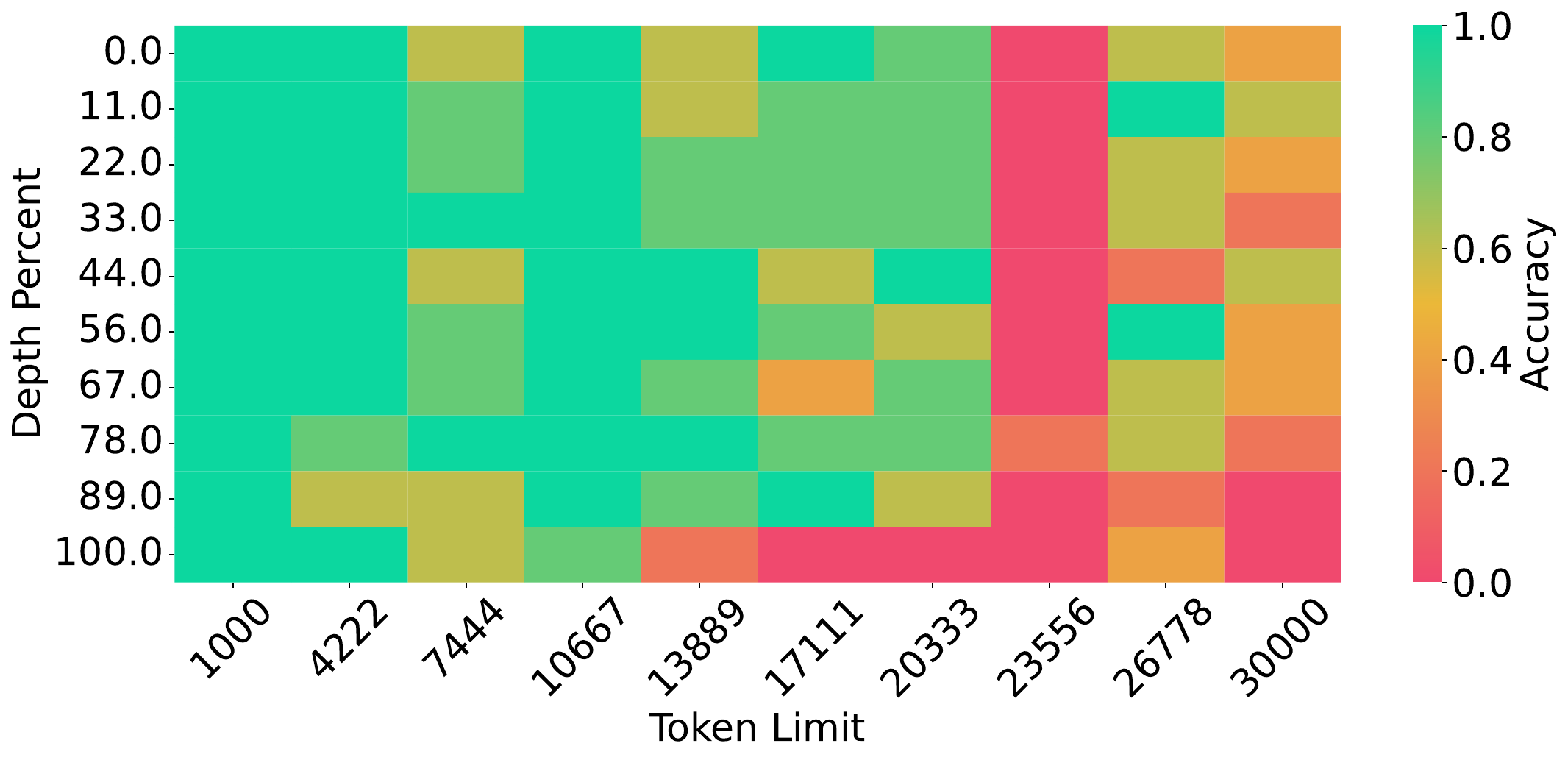}
         \caption{H2O-FastCAOTE}
     \end{subfigure}
     \hfill
     \begin{subfigure}[b]{0.33\linewidth}
         \centering
         \includegraphics[width=\linewidth]{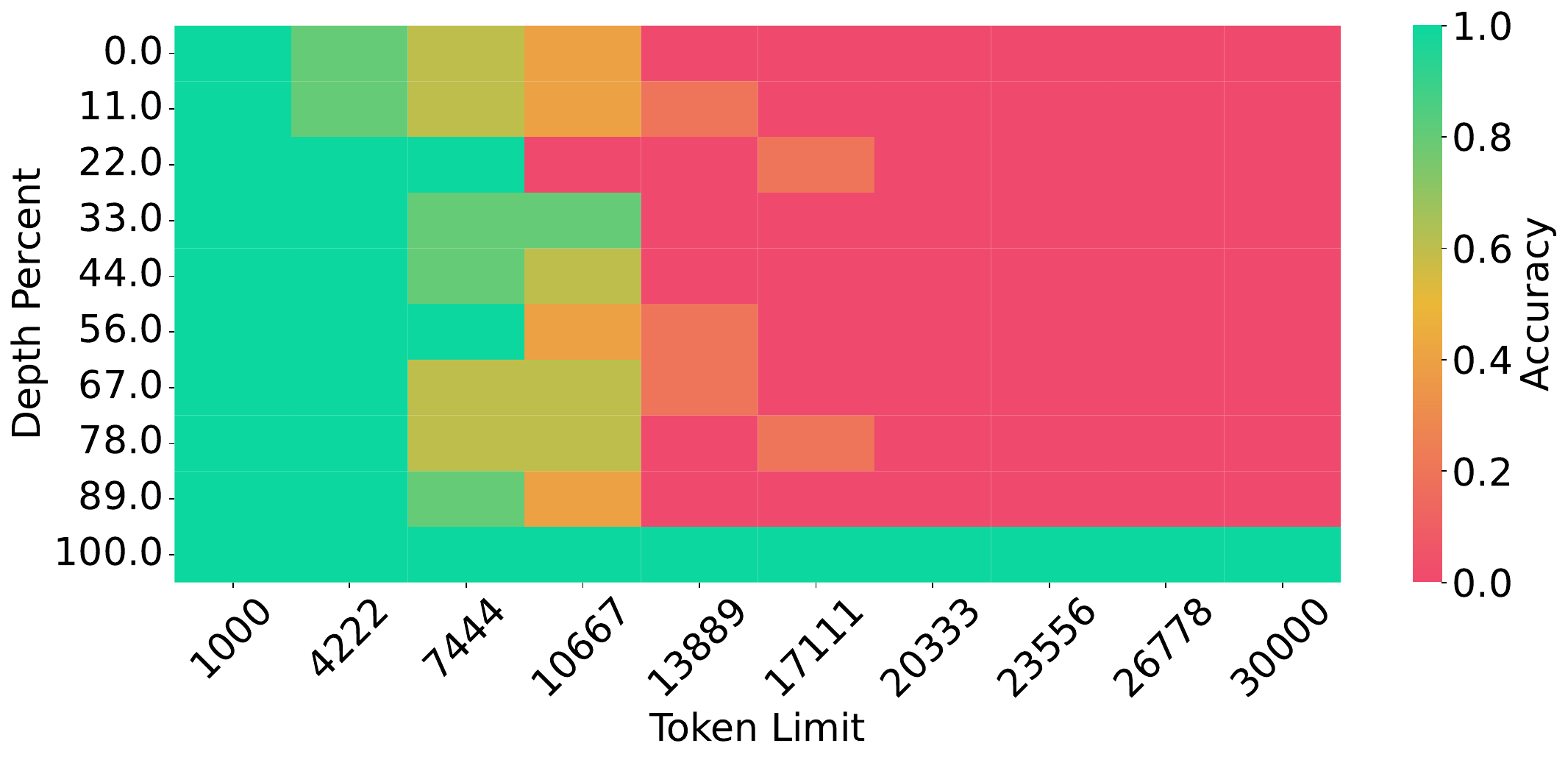}
         \caption{TOVA-FastCAOTE}
     \end{subfigure}
     \begin{subfigure}[b]{0.33\linewidth}
         \centering
         \includegraphics[width=\linewidth]{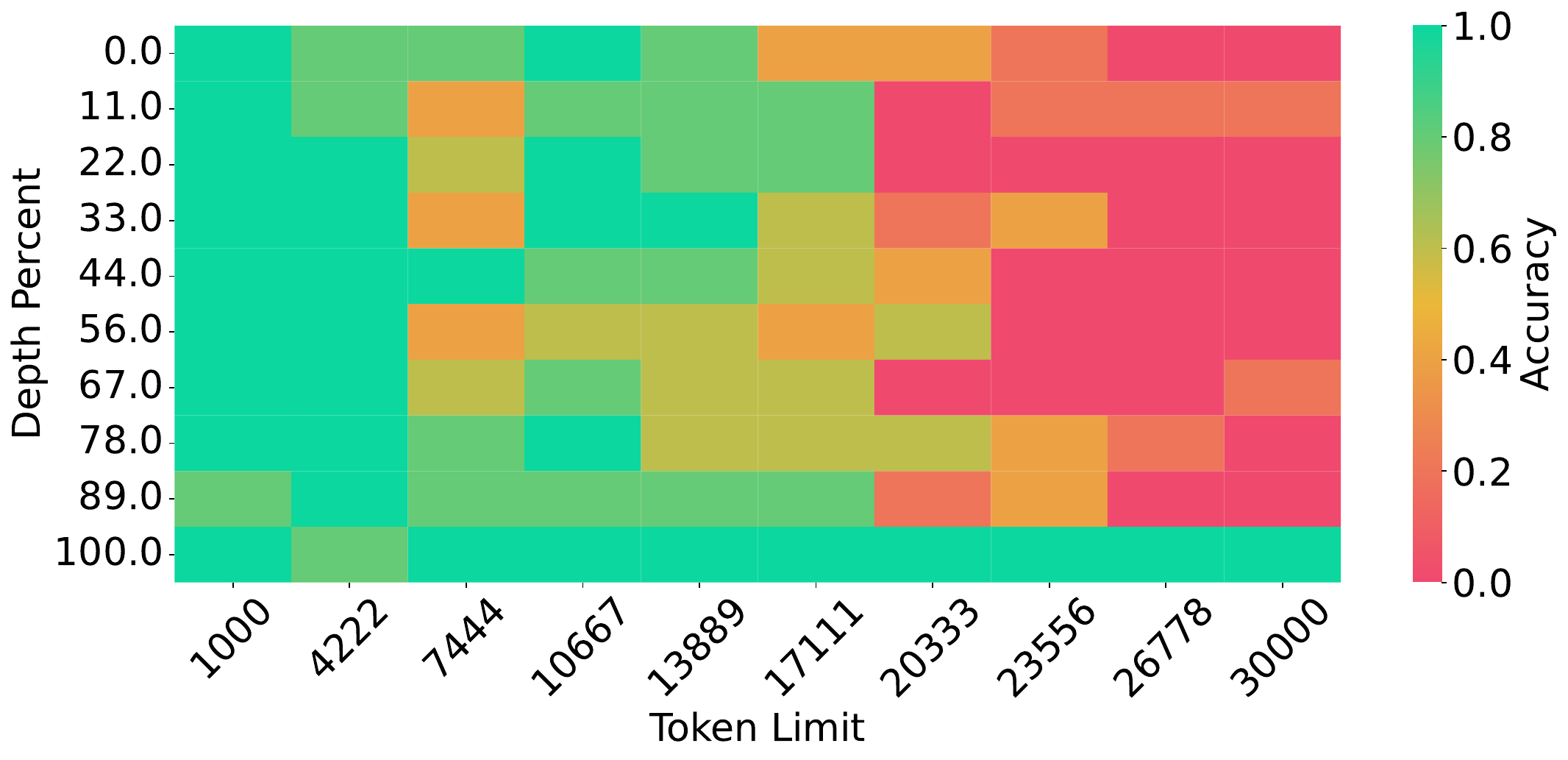}
         \caption{SnapKV-FastCAOTE}
    \end{subfigure}    
    \caption{Needle-In-A-Haystack accuracies of Llama 3.1-8B-Instruct with token eviction with $6$k cache budget.}
    \label{fig:needle_6k_llama8b}
\end{figure*}
\begin{table*}[t]
\caption{\textbf{Needle-in-haystack accuracy} for Llama 3.2-3B/3.1-8B-Instruct using baseline eviction methods with(out) \textit{CAOTE}. Higher is better, maximum accuracy is $1.0$.
}
\vspace{-2mm}
\label{table:needle_wide}
\begin{center}\resizebox{0.8\textwidth}{!}{
\begin{tabular}{cccccccccc}

\toprule
\multirow{2}{*}{Budget} & \multicolumn{3}{c}{H2O} &  \multicolumn{3}{c}{TOVA} &  \multicolumn{3}{c}{SnapKV} 
\\
\cmidrule(lr){2-4}
\cmidrule(lr){5-7}
\cmidrule(lr){8-10}
& & +CAOTE & +FastCAOTE & & +CAOTE & +FastCAOTE & & +CAOTE & +FastCAOTE
\\
\midrule
\multicolumn{10}{c}{Llama 3.1-8B-Instruct}
\\
\midrule
2k & 0.174 & \textbf{0.270} & 0.264  & 0.196 & 0.204 & 0.202  & 0.214 & 0.226 & 0.242
\\ 
4k & 0.330 & 0.538 & \textbf{0.568}  & 0.286 & 0.298 & 0.292 & 0.360 & 0.392 & 0.420
\\
6k & 0.544 & \textbf{0.698} & 0.676 & 0.370 & 0.402 & 0.396 & 0.490 & 0.550 & 0.580
\\
\midrule
\multicolumn{10}{c}{Llama 3.2-3B-Instruct}
\\
\midrule
2k & 0.104& 0.160 & \textbf{0.172} & \textbf{0.172} & 0.150 & 0.166 & 0.154 & \textbf{0.172} & 0.168
\\
4k & 0.198 & 0.262 & \textbf{0.294} & 0.220 & 0.232 & 0.232 & 0.226 & 0.222 & 0.232
\\
6k & 0.258 & 0.308 & \textbf{0.322} & 0.258 & 0.278 & 0.270 & 0.272 & 0.264 & 0.312
\\
8k & 0.324 & 0.414 & \textbf{0.404} & 0.338 & 0.364 & 0.344 & 0.342 & 0.336 & 0.366
\\
\bottomrule
\end{tabular}
}
\vspace{-5mm}
\end{center}
\end{table*}
Lastly, we run extensive experiments on Needle-In-A-Haystack benchmark \cite{liu2024lost,kamradt2023needle} and show quantitative results in \cref{table:needle_wide} and visualizations for $6$k budget for \texttt{Llama3.1-8B-Instruct} in \cref{fig:needle_6k_llama8b}. We observe in \cref{table:needle_wide} that \textit{H2O-FastCAOTE} performs best for all budgets with \texttt{Llama3.2-3B-Instruct}.
When using a budget $4$k with \texttt{Llama3.1-8B-Instruct}, \textit{CAOTE} boosted \textit{H2O} outperforms \textit{TOVA, SnapKV} as well. \textit{H2O-CAOTE} performs best for \texttt{Llama3.1-8B-Instruct} with budget = \{$2$k, $6$k\} and \textit{H2O-FastCAOTE} performs best for$4$k budget for \texttt{Llama3}.1-8b. The gains in precision are especially high for the $4$ k budget for the \texttt{Llama3.1-8B-Instruct}, with an increase of up to $30-60 \%$. 
We can see in \cref{fig:needle_6k_llama8b} that \textit{CAOTE} improves the state-of-the-art eviction method and is able to predict beyond their budget constraints. Results for \texttt{Qwen2.5} models are shown in \cref{table:needle_qwen_wide} in \cref{subsec:appendix_results_needle}.

\section{Related Work}
\label{section:related_work}

\paragraph{Sparse and Efficient Attention}
Sparse or efficient attention based methods result in mitigating the computation load and saving memory consumption by using efficient linear attentions  \cite{katharopoulos2020transformers}. Additionally, there are KV compression methods which don't evict any tokens as post eviction the token is not retrievable, \cite{dong2024get} proposes to keep important tokens based on attention score in cache while combining the evicted tokens via linear attention into single embedding. Landmark attention injects learnable special tokens between chunks of tokens and access past tokens in chunks instead of individually. Lastly, there are better architectures based with constant KV memory which outperform linear attention attentions \cite{mohtashami2023landmark}. However, all these methods require either from-scratch training or fine-tuning. 

\paragraph{KV Cache Eviction}
At the extreme end of efficient KV cache management, token eviction methods have been extensively studied. Leveraging the sparsity of attention in LLMs \cite{xiaoefficient, pmlr-v202-sheng23a, chen2021scatterbrain}, these methods determine the importance of KV pairs using (learned) rules and retain those with high scores to approximate the attention output.

StreamingLLM \cite{xiaoefficient} observes an attention sink phenomenon, where the first few tokens receive the majority of attention weights. It proposes SinkAttention, which prioritizes retaining initial tokens while applying sliding window-based attention. As seen in Appendix \cref{table:longbench_reduced_llama_68}, SinkAttention with sliding window-based budgeted key-value cache management performs poor than recent methods with \textit{CAOTE}. 
Other methods such as \textit{H2O} \cite{zhang2024h2o}, \textit{TOVA} \cite{oren2024transformers}, \textit{SnapKV} \cite{li2024snapkv}, and \textit{RoCO} \cite{ren2024efficacy} retain tokens with high attention scores using various algorithmic strategies — including preserving first/last tokens or applying smoothing to attention scores.

While these approaches primarily rely on attention scores to assess token importance, \textit{CAOTE} introduces an orthogonal scoring metric that estimates the impact of value vectors on approximating attention outputs. This value-centric perspective complements existing attention-based scoring methods and can be integrated with them to enhance eviction performance.

Recent works such as \textit{Quest} \citep{tang2024quest} propose using approximate attention scores for efficient cache management. \textit{CAOTE} can be combined with Quest by approximating value vectors (e.g., via min/max pooling) and computing approximate attention outputs to guide eviction decisions. Similarly, \textit{CaM} \cite{zhang2024cam} is based solely on attention scores; \textit{CAOTE} can enhance this by incorporating value vector information. Notably, both Quest and CaM report results only on \texttt{LLaMA2}, without comparisons to other eviction methods, limiting their relevance to current frontier models.

In contrast, we performed comparison with \textit{SnapKV}, developed around the same time, demonstrates state-of-the-art performance on modern models like \texttt{LLaMA3}, making it more representative of current deployment scenarios. \textit{CAOTE} complements SnapKV by introducing a value-centric scoring mechanism that can be integrated with attention-based heuristics to further improve eviction strategies.

Other lines of work focus on layer-wise budget optimization \cite{feng2024ada}, selecting top-K nodes across heads, or managing memory based on token characteristics \cite{ge2023model}, often using baseline eviction strategies like \textit{H2O}. \textit{CAOTE} is highly flexible and can be integrated with these approaches to achieve further performance gains.
Finally, \textit{DuoAttention}~\cite{xiao2024duoattention} explore orthogonal directions: addresses attention head selection. This method is complementary to \textit{CAOTE}, which focuses specifically on token-level eviction.

\textbf{In summary}, while recent methods like \textit{SnapKV}, \textit{Quest}, and \textit{CaM} explore efficient token eviction through attention-based heuristics, \textit{CAOTE} introduces a fundamentally new direction by leveraging value-centric scoring. Its modular design allows seamless integration with both attention-based and approximate methods, making it highly adaptable across model families — including frontier models like \texttt{LLaMA3} and \texttt{Qwen2.5}.


\section{Conclusion}
\label{sec:conclusion}
We propose a post-training KV cache eviction method that can be seamlessly integrated with any existing eviction strategies. Our approach, \textit{CAOTE}, introduces an optimization objective aimed at minimizing the alteration in attention output when evicting a token. This objective ensures the incorporation of both attention scores and value vectors in the eviction decision process. Our formulation allows for the parallel computation of the \textit{CAOTE} score for all tokens. Additionally, we present an efficient variant, \textit{FastCAOTE}. Through extensive evaluations across various downstream tasks, we demonstrate that eviction methods equipped with \textit{CAOTE} consistently deliver superior performance. 

\section*{Impact Statement}
This paper presents work whose goal is to advance the field of Machine Learning. There are many potential societal consequences of our work, none of which we feel must be specifically highlighted here.


\bibliography{reference}

\begin{thebibliography}{32}
\providecommand{\natexlab}[1]{#1}
\providecommand{\url}[1]{\texttt{#1}}
\expandafter\ifx\csname urlstyle\endcsname\relax
  \providecommand{\doi}[1]{doi: #1}\else
  \providecommand{\doi}{doi: \begingroup \urlstyle{rm}\Url}\fi

\bibitem[Agrawal et~al.(2023)Agrawal, Panwar, Mohan, Kwatra, Gulavani, and Ramjee]{agrawal2023sarathi}
Agrawal, A., Panwar, A., Mohan, J., Kwatra, N., Gulavani, B.~S., and Ramjee, R.
\newblock Sarathi: Efficient llm inference by piggybacking decodes with chunked prefills.
\newblock \emph{arXiv preprint arXiv:2308.16369}, 2023.

\bibitem[Bai et~al.(2024)Bai, Lv, Zhang, Lyu, Tang, Huang, Du, Liu, Zeng, Hou, Dong, Tang, and Li]{bai-etal-2024-longbench}
Bai, Y., Lv, X., Zhang, J., Lyu, H., Tang, J., Huang, Z., Du, Z., Liu, X., Zeng, A., Hou, L., Dong, Y., Tang, J., and Li, J.
\newblock {L}ong{B}ench: A bilingual, multitask benchmark for long context understanding.
\newblock In Ku, L.-W., Martins, A., and Srikumar, V. (eds.), \emph{Proceedings of the 62nd Annual Meeting of the Association for Computational Linguistics (Volume 1: Long Papers)}, pp.\  3119--3137, Bangkok, Thailand, August 2024. Association for Computational Linguistics.
\newblock \doi{10.18653/v1/2024.acl-long.172}.
\newblock URL \url{https://aclanthology.org/2024.acl-long.172}.

\bibitem[Chen et~al.(2021)Chen, Dao, Winsor, Song, Rudra, and R{\'e}]{chen2021scatterbrain}
Chen, B., Dao, T., Winsor, E., Song, Z., Rudra, A., and R{\'e}, C.
\newblock Scatterbrain: Unifying sparse and low-rank attention.
\newblock \emph{Advances in Neural Information Processing Systems}, 34:\penalty0 17413--17426, 2021.

\bibitem[Coenen et~al.(2021)Coenen, Davis, Ippolito, Reif, and Yuan]{coenen2021wordcraft}
Coenen, A., Davis, L., Ippolito, D., Reif, E., and Yuan, A.
\newblock Wordcraft: A human-ai collaborative editor for story writing.
\newblock \emph{arXiv preprint arXiv:2107.07430}, 2021.

\bibitem[Dong et~al.(2024)Dong, Yang, Zhang, Wang, Chi, and Chen]{dong2024get}
Dong, H., Yang, X., Zhang, Z., Wang, Z., Chi, Y., and Chen, B.
\newblock Get more with less: Synthesizing recurrence with kv cache compression for efficient llm inference.
\newblock \emph{arXiv preprint arXiv:2402.09398}, 2024.

\bibitem[Dubey et~al.(2024)Dubey, Jauhri, Pandey, Kadian, Al-Dahle, Letman, Mathur, Schelten, Yang, Fan, et~al.]{dubey2024llama}
Dubey, A., Jauhri, A., Pandey, A., Kadian, A., Al-Dahle, A., Letman, A., Mathur, A., Schelten, A., Yang, A., Fan, A., et~al.
\newblock The llama 3 herd of models.
\newblock \emph{arXiv preprint arXiv:2407.21783}, 2024.

\bibitem[Feng et~al.(2024)Feng, Lv, Cao, Xie, and Zhou]{feng2024ada}
Feng, Y., Lv, J., Cao, Y., Xie, X., and Zhou, S.~K.
\newblock Ada-kv: Optimizing kv cache eviction by adaptive budget allocation for efficient llm inference.
\newblock \emph{arXiv preprint arXiv:2407.11550}, 2024.

\bibitem[Ge et~al.(2023)Ge, Zhang, Liu, Zhang, Han, and Gao]{ge2023model}
Ge, S., Zhang, Y., Liu, L., Zhang, M., Han, J., and Gao, J.
\newblock Model tells you what to discard: Adaptive kv cache compression for llms.
\newblock \emph{arXiv preprint arXiv:2310.01801}, 2023.

\bibitem[Holmes et~al.(2024)Holmes, Tanaka, Wyatt, Awan, Rasley, Rajbhandari, Aminabadi, Qin, Bakhtiari, Kurilenko, et~al.]{holmes2024deepspeed}
Holmes, C., Tanaka, M., Wyatt, M., Awan, A.~A., Rasley, J., Rajbhandari, S., Aminabadi, R.~Y., Qin, H., Bakhtiari, A., Kurilenko, L., et~al.
\newblock Deepspeed-fastgen: High-throughput text generation for llms via mii and deepspeed-inference.
\newblock \emph{arXiv preprint arXiv:2401.08671}, 2024.

\bibitem[Kamradt(2023)]{kamradt2023needle}
Kamradt, G.
\newblock Needle in a haystack - pressure testing llms.
\newblock GitHub repository, 2023.
\newblock URL \url{https://github.com/gkamradt/LLMTest_NeedleInAHaystack}.

\bibitem[Katharopoulos et~al.(2020)Katharopoulos, Vyas, Pappas, and Fleuret]{katharopoulos2020transformers}
Katharopoulos, A., Vyas, A., Pappas, N., and Fleuret, F.
\newblock Transformers are rnns: Fast autoregressive transformers with linear attention.
\newblock In \emph{International conference on machine learning}, pp.\  5156--5165. PMLR, 2020.

\bibitem[Kry{\'s}ci{\'n}ski et~al.(2021)Kry{\'s}ci{\'n}ski, Rajani, Agarwal, Xiong, and Radev]{kryscinski2021booksum}
Kry{\'s}ci{\'n}ski, W., Rajani, N., Agarwal, D., Xiong, C., and Radev, D.
\newblock Booksum: A collection of datasets for long-form narrative summarization.
\newblock \emph{arXiv preprint arXiv:2105.08209}, 2021.

\bibitem[Li et~al.(2024)Li, Huang, Yang, Venkitesh, Locatelli, Ye, Cai, Lewis, and Chen]{li2024snapkv}
Li, Y., Huang, Y., Yang, B., Venkitesh, B., Locatelli, A., Ye, H., Cai, T., Lewis, P., and Chen, D.
\newblock Snapkv: Llm knows what you are looking for before generation.
\newblock \emph{arXiv preprint arXiv:2404.14469}, 2024.

\bibitem[Liao et~al.(2024)Liao, Wang, Li, Wang, Huang, and Jin]{liao2024doclayllm}
Liao, W., Wang, J., Li, H., Wang, C., Huang, J., and Jin, L.
\newblock Doclayllm: An efficient and effective multi-modal extension of large language models for text-rich document understanding.
\newblock \emph{arXiv preprint arXiv:2408.15045}, 2024.

\bibitem[Liu et~al.(2024)Liu, Lin, Hewitt, Paranjape, Bevilacqua, Petroni, and Liang]{liu2024lost}
Liu, N.~F., Lin, K., Hewitt, J., Paranjape, A., Bevilacqua, M., Petroni, F., and Liang, P.
\newblock Lost in the middle: How language models use long contexts.
\newblock \emph{Transactions of the Association for Computational Linguistics}, 12:\penalty0 157--173, 2024.

\bibitem[Mohtashami \& Jaggi(2023)Mohtashami and Jaggi]{mohtashami2023landmark}
Mohtashami, A. and Jaggi, M.
\newblock Landmark attention: Random-access infinite context length for transformers.
\newblock \emph{arXiv preprint arXiv:2305.16300}, 2023.

\bibitem[Oren et~al.(2024)Oren, Hassid, Adi, and Schwartz]{oren2024transformers}
Oren, M., Hassid, M., Adi, Y., and Schwartz, R.
\newblock Transformers are multi-state rnns.
\newblock \emph{arXiv preprint arXiv:2401.06104}, 2024.

\bibitem[Pope et~al.(2023)Pope, Douglas, Chowdhery, Devlin, Bradbury, Heek, Xiao, Agrawal, and Dean]{pope2023efficiently}
Pope, R., Douglas, S., Chowdhery, A., Devlin, J., Bradbury, J., Heek, J., Xiao, K., Agrawal, S., and Dean, J.
\newblock Efficiently scaling transformer inference.
\newblock \emph{Proceedings of Machine Learning and Systems}, 5:\penalty0 606--624, 2023.

\bibitem[Qin et~al.(2025)Qin, Cao, Lin, Hu, Fan, Cheng, Lin, and Li]{qin2025cake}
Qin, Z., Cao, Y., Lin, M., Hu, W., Fan, S., Cheng, K., Lin, W., and Li, J.
\newblock Cake: Cascading and adaptive kv cache eviction with layer preferences.
\newblock \emph{arXiv preprint arXiv:2503.12491}, 2025.

\bibitem[Ren \& Zhu(2024)Ren and Zhu]{ren2024efficacy}
Ren, S. and Zhu, K.~Q.
\newblock On the efficacy of eviction policy for key-value constrained generative language model inference.
\newblock \emph{arXiv preprint arXiv:2402.06262}, 2024.

\bibitem[Robinson et~al.(2022)Robinson, Rytting, and Wingate]{robinson2022leveraging}
Robinson, J., Rytting, C.~M., and Wingate, D.
\newblock Leveraging large language models for multiple choice question answering.
\newblock \emph{arXiv preprint arXiv:2210.12353}, 2022.

\bibitem[Sheng et~al.(2023)Sheng, Zheng, Yuan, Li, Ryabinin, Chen, Liang, Re, Stoica, and Zhang]{pmlr-v202-sheng23a}
Sheng, Y., Zheng, L., Yuan, B., Li, Z., Ryabinin, M., Chen, B., Liang, P., Re, C., Stoica, I., and Zhang, C.
\newblock {F}lex{G}en: High-throughput generative inference of large language models with a single {GPU}.
\newblock In Krause, A., Brunskill, E., Cho, K., Engelhardt, B., Sabato, S., and Scarlett, J. (eds.), \emph{Proceedings of the 40th International Conference on Machine Learning}, volume 202 of \emph{Proceedings of Machine Learning Research}, pp.\  31094--31116. PMLR, 23--29 Jul 2023.
\newblock URL \url{https://proceedings.mlr.press/v202/sheng23a.html}.

\bibitem[Tang et~al.(2024)Tang, Zhao, Zhu, Xiao, Kasikci, and Han]{tang2024quest}
Tang, J., Zhao, Y., Zhu, K., Xiao, G., Kasikci, B., and Han, S.
\newblock Quest: Query-aware sparsity for efficient long-context llm inference.
\newblock \emph{arXiv preprint arXiv:2406.10774}, 2024.

\bibitem[Thoppilan et~al.(2022)Thoppilan, De~Freitas, Hall, Shazeer, Kulshreshtha, Cheng, Jin, Bos, Baker, Du, et~al.]{thoppilan2022lamda}
Thoppilan, R., De~Freitas, D., Hall, J., Shazeer, N., Kulshreshtha, A., Cheng, H.-T., Jin, A., Bos, T., Baker, L., Du, Y., et~al.
\newblock Lamda: Language models for dialog applications.
\newblock \emph{arXiv preprint arXiv:2201.08239}, 2022.

\bibitem[Xiao et~al.(2024{\natexlab{a}})Xiao, Tang, Zuo, Guo, Yang, Tang, Fu, and Han]{xiao2024duoattention}
Xiao, G., Tang, J., Zuo, J., Guo, J., Yang, S., Tang, H., Fu, Y., and Han, S.
\newblock Duoattention: Efficient long-context llm inference with retrieval and streaming heads.
\newblock \emph{arXiv preprint arXiv:2410.10819}, 2024{\natexlab{a}}.

\bibitem[Xiao et~al.(2024{\natexlab{b}})Xiao, Tian, Chen, Han, and Lewis]{xiaoefficient}
Xiao, G., Tian, Y., Chen, B., Han, S., and Lewis, M.
\newblock Efficient streaming language models with attention sinks.
\newblock In \emph{The Twelfth International Conference on Learning Representations}, 2024{\natexlab{b}}.

\bibitem[Xiao et~al.(2023)Xiao, Wu, Guo, Li, Zhang, Qin, and Liu]{xiao2023survey}
Xiao, Y., Wu, L., Guo, J., Li, J., Zhang, M., Qin, T., and Liu, T.-y.
\newblock A survey on non-autoregressive generation for neural machine translation and beyond.
\newblock \emph{IEEE Transactions on Pattern Analysis and Machine Intelligence}, 45\penalty0 (10):\penalty0 11407--11427, 2023.

\bibitem[Yang et~al.(2024)Yang, Yang, Zhang, Hui, Zheng, Yu, Li, Liu, Huang, Wei, et~al.]{yang2024qwen2}
Yang, A., Yang, B., Zhang, B., Hui, B., Zheng, B., Yu, B., Li, C., Liu, D., Huang, F., Wei, H., et~al.
\newblock Qwen2. 5 technical report.
\newblock \emph{arXiv preprint arXiv:2412.15115}, 2024.

\bibitem[Zhang et~al.(2024{\natexlab{a}})Zhang, Ladhak, Durmus, Liang, McKeown, and Hashimoto]{zhang2024benchmarking}
Zhang, T., Ladhak, F., Durmus, E., Liang, P., McKeown, K., and Hashimoto, T.~B.
\newblock Benchmarking large language models for news summarization.
\newblock \emph{Transactions of the Association for Computational Linguistics}, 12:\penalty0 39--57, 2024{\natexlab{a}}.

\bibitem[Zhang et~al.(2023)Zhang, Lv, and Yang]{zhang2023adaptive}
Zhang, X., Lv, Z., and Yang, Q.
\newblock Adaptive attention for sparse-based long-sequence transformer.
\newblock In \emph{Findings of the Association for Computational Linguistics: ACL 2023}, pp.\  8602--8610, 2023.

\bibitem[Zhang et~al.(2024{\natexlab{b}})Zhang, Du, Luo, Zhong, Zhang, Liu, and Ji]{zhang2024cam}
Zhang, Y., Du, Y., Luo, G., Zhong, Y., Zhang, Z., Liu, S., and Ji, R.
\newblock Cam: Cache merging for memory-efficient llms inference.
\newblock In \emph{Forty-first international conference on machine learning}, 2024{\natexlab{b}}.

\bibitem[Zhang et~al.(2024{\natexlab{c}})Zhang, Sheng, Zhou, Chen, Zheng, Cai, Song, Tian, R{\'e}, Barrett, et~al.]{zhang2024h2o}
Zhang, Z., Sheng, Y., Zhou, T., Chen, T., Zheng, L., Cai, R., Song, Z., Tian, Y., R{\'e}, C., Barrett, C., et~al.
\newblock H2o: Heavy-hitter oracle for efficient generative inference of large language models.
\newblock \emph{Advances in Neural Information Processing Systems}, 36, 2024{\natexlab{c}}.

\end{thebibliography}
\bibliographystyle{icml2025}

\appendix

\section{Limitation}
\textit{CAOTE} is a myopic (greedy) strategy and its scoring framework based on the assumption of evicting $1$ token per iteration. This assumption breaks during prefilling stage, however, taking into account change in attention output due to multi-token eviction is non-trivial. Fortunately, we observe that even assuming multi-token eviction independently (without considering the effect of other tokens being evicted), \textit{CAOTE} is still able to give boost in performance for all tasks. Due to this reason \textit{CAOTE}'s performance might further improve with smaller prompt filling block-size.

\section{Extending \textit{CAOTE} to Multi-token Eviction}
\label{appendix:multi_token_caote}
We start with simple example, without loss of generality consider evicting two tokens in the order $i=1, j=2$. From \eqref{eq:caote_score_based_on_tova}
single token eviction error for token $1$ is
\begin{equation}
    c_{1} = \frac{\alpha_{1}}{1-\alpha_{1}}||X_{\text{attn}} - v_{1}||_{2}
\end{equation}
After evicting token $1$, the updated attention output becomes:
\begin{equation}
    X_{\text{attn},1}^{'}=\Sigma_{i=2}^{n+1}\alpha_{i}^{'}v_{i}, \, \text{where } \alpha_{i}^{'} = \frac{\alpha_{i}}{1-\alpha_{1}}
    \label{eq:x_attn_1_dash}
\end{equation}
Now, evicting token $2$ from updated attention output above yields
\begin{align}
    & X_{\text{attn},[1,2]}^{''} = \Sigma_{i=3}^{n+1}\alpha_{i}^{''}v_{i}, 
    \\
    & \text{where } \alpha_{i}^{''}=\frac{\alpha_{i}^{'}}{1-\alpha_{2}^{'}}=\frac{\alpha_{i}}{1-\alpha_{1}-\alpha_{2}}
\end{align}
This leads to the following insights:
\begin{itemize}
    \item The updated attention output can be expressed as 
    \begin{equation}
        X_{\text{attn},[1,2]}^{''} = \frac{1}{1-\alpha_{2}^{'}}(X_{\text{attn},1}^{'} - \alpha_{2}^{'}v_{2})
    \end{equation}
    \item Eviction error of removing token $2$ after removing token $1$ is
    \begin{equation}
        c_{[1,2]} = ||X_{\text{attn},1}^{'}-X_{\text{attn},[1,2]}^{''}||_{2}
    \end{equation}
    \item Substituting the expressions above, we obtain a \textbf{closed-form joint eviction score}:
    \begin{equation}
        c_{[1,2]} = \frac{1}{1-\alpha_{1}-\alpha_{2}}||\alpha_{1}(X_{\text{attn}}-v_{1}) + \alpha_{2}(X_{\text{attn}}-v_{2})||_{2}
    \end{equation}
\end{itemize}

This formulation generalizes to evicting $m$ tokens jointly (assuming first $m$ tokens here without loss of generality)
\begin{equation}
    c_{[1,\dots, m]} = \frac{1}{1-\Sigma_{i=1}^{m}\alpha_{i}}||\Sigma_{i=1}^{m}\alpha_{i}(X_{\text{attn}}-v_{i})||_{2}
\end{equation}
This expression highlights the combinatorial nature of multi-token eviction: for $n$ tokens with $m$-token eviction, there are $\binom{n}{m}$ possible combinations. This makes exact computation intractable for large $m$, motivating the need for approximate or greedy strategies.

Importantly, existing token eviction methods typically rely only on attention scores and assume independence between tokens. As a result, the relative ranking of tokens remains unchanged even after evictions, which limits their effectiveness in multi-token settings.

In contrast, CAOTE explicitly models the interplay between attention scores and value vectors, capturing how evicting one token affects the contribution of others. This introduces interdependencies between tokens during eviction, which are critical for accurate multi-token decisions.

\section{Inference Latency}
\label{appendix:inference_latency}
We analyze the computational overhead of \textit{CAOTE} and \textit{FastCAOTE} during both prefill and generation phases. The overhead is minimal, especially for \textit{FastCAOTE}.

Let $s$ be sequence length, $d$ the hidden dimension, $d_{KV}$ the KV hidden dimension, $d_{FFN}$ the intermediate dimension, $d_{V}$ the vocabulary size, and $L$ the number of hidden layers. The total floating-point Operation count are as follows:
\begin{align}
    &\text{prefill flops } : 2Lsd(d+d_{KV}+d_{FFN}+2s+6) + dd_{V}
    \\
    &\text{generation flops }: 2Ld(d+d_{KV}+d_{FFN}+2s+6) + dd_{V}
    \\
    &\text{\textit{CAOTE} flops}: Ls(7d + 3)
    \\
    &\text{\textit{FastCAOTE} flops}: Ls(4d + 3) + L
\end{align}
We compute the relative overhead for \texttt{Llama3}.1-8B model with different sequence lengths and show the ratio between prefill (and generation) flops to \textit{CAOTE} flops in \cref{tab:prefill_inference_latency}, \ref{tab:generation_inference_latency} 
\begin{table}[]
    \centering
    \begin{tabular}{c|cc}
        Sequence Length & \multicolumn{2}{c}{Ratio}
        \\
        & with \textit{CAOTE} & with \textit{FastCAOTE}
        \\
        \midrule
        4k & 1.6e-4 & 8.9e-5
        \\
        8k & 9.25e-5 & 5.28e-5
        \\
        32k & 6.45e-5 & 3.69e-5
        \\
        \midrule
    \end{tabular}
    \caption{Ratio between prefill flops and \textit{(Fast)CAOTE} flops. \textit{CAOTE} adds minimal overhead during prefill phase.}
    \label{tab:prefill_inference_latency}
\end{table}

\begin{table}[]
    \centering
    \begin{tabular}{c|cc}
        Sequence Length & \multicolumn{2}{c}{Ratio}
        \\
        & with \textit{CAOTE} & with \textit{FastCAOTE}
        \\
        \midrule
        512 & 0.08 & 0.046
        \\
        1024 & 0.15 & 0.087
        \\
        \midrule
    \end{tabular}
    \caption{Ratio between generation flops and \textit{(Fast)CAOTE} flops. \textit{CAOTE} adds minimal overhead during prefill phase, especially \textit{FastCAOTE}}
    \label{tab:generation_inference_latency}
\end{table}

\section{Additional Proofs}
\label{sec:appendix_proofs}
\subsection{H2O scores}
\label{subsec:h2o_scores}
\begin{theorem}
    Given H2O scores for $b+1$ tokens as $[h_{1}, \dots, h_{b+1}]$, the summation of all $h_{i}$, $\forall i\in\{1, \dots, b+1\}$ is greater than $1$ 
    \begin{equation}
     \Sigma_{i=1}^{b+1}h_{i} > 1   
    \label{thm:h2o_scores}
    \end{equation}
\end{theorem}
\begin{proof}
Assuming that only $b+1$ tokens are present and are propagated through the model at the same time. The causal attention mask $A \in [0,1]^{b+1 \times b+1}$, will have all entries on upper triangle excluding diagonal is $0$. The first token will attend to itself have attention score as $1$
\begin{align}
    A_{1,1} =& 1
    \\
    A_{1,>1} =& 0
\end{align}

H2O score for token $1$ is defined as the sum of attention score to token $1$ by all future tokens:
\begin{equation}
    h_{1} = A_{1,1} + A_{2,1} + \dots + A_{b+1,1} = \Sigma_{i=1}^{b+1}A_{i,1}
    \label{eq:h2o_1}
\end{equation}    
In general for a token $j$, the H2O score is defined as
\begin{align}
    h_{j} =& A_{j,j} + A_{j+1,j} + \dots + A_{b+1,j} 
    \\
    = & \underbrace{A_{1,j} + \dots + A_{j-1,j}}_{=0 \text{ causal mask}} + A_{j,j} \dots + A_{b+1,j}
    \\
    = & \Sigma_{i=1}^{b+1}A_{i,j}
    \label{eq:h2o_gen}
\end{align}

Using Eq. \eqref{eq:h2o_gen} and summing for all H2O scores, we get
\begin{align}
    \Sigma_{j=1}^{b+1}h_{j} =& \Sigma_{j=1}^{b+1}\Sigma_{i=1}^{b+1}A_{i,j}
    \\
    =& \Sigma_{i=1}^{b+1}\Sigma_{j=1}^{b+1}A_{i,j}
    \\
    =& \Sigma_{i=1}^{b+1}(\underbrace{\Sigma_{j=1}^{i}A_{i,j}}_{=1 \text{ due to softmax}} + \underbrace{\Sigma_{j=j+1}^{b+1}A_{i,j}}_{=0 \text{ for causal mask}})
    \\
    = & b+1 > 1
\end{align}
Hence proved.

\end{proof}

\subsection{Relation of CAOTE score to output logits}
\label{subsec:appendix_logits}
We show that evicting token based on \textit{CAOTE} score can lead to smaller discrepancy in final logits which affect the downstream performance. As \textit{CAOTE} score is the eviction error during generation phase, we instead show the relation between eviction error and logits. We start by showing for a single attention layer (single head) based network, its extension to multiple heads and, finally a transformer layer (self-attention and feed-forward-network). For simplicity we ignore layer-norms. Some definitions which will used are given below. 
We assume a budget of $b$, with current token sequence having $b+1$ tokens (superscript denotes layer):
\begin{align}
    X^{0} \triangleq [x_{1},x_{2},\dots,x_{b+1}], \, X^{0}_{b+1}\triangleq x_{b+1} 
\end{align}
The attention output for a sequence of $b+1$ tokens is (for layer $m$)
\begin{align}
    X_{A,b+1}^{m} \triangleq \Sigma_{j=1}^{b+1}\alpha_{j}^{m}W_{v}^{m}X_{j}^{m}
\end{align}
The logits with and without eviction for token $j$ are defined as $l_{j}, \hat{l}_{j}$ respectively.

\textbf{(Case 1) Single self-attention layer with single head}: The logits for dense baseline is:
\begin{align}
    X^{1}_{b+1} =& X^{0}_{b+1} + W_{O}X_{A,b+1}^{0}
    \\
    l_{b+1} =& W_{H}X^{1}_{b+1} = W_{H}(X^{0}_{b+1} + W_{O}\Sigma_{j}\alpha_{j}W_{V}X^{0}_{j})
\end{align}
where $W_{H}, W_{O}, W_{V}$ are the LM-head, output projection, and value projection respectively. $X^{1}_{b+1}$ is the output after the residual connection. 

The logits with eviction will have a perturbation due to error in attention output (CAOTE score or eviction error), and is given as:
\begin{align}
    \hat{X}^{1}_{b+1} =& X^{0}_{b+1} + W_{O}X_{A,b+1}^{0} + \textcolor{red}{W_{O}\Delta_{A}^{0}}
    \\
    \hat{l}_{b+1} =& W_{H}\hat{X}_{b+1}^{1}
    \\
                =& W_{H}(X^{0}_{b+1} + W_{O}X_{A,b+1}+\textcolor{red}{W_{O}\Delta_{A}^{0}}) 
    \\
                =& l_{b+1} + \textcolor{red}{\Delta_{l,b+1}}
\end{align}
where the logit error is $\Delta_{l,b+1} = W_{H}W_{O}\Delta_{A,b+1}$, $\Delta_{A,b+1} = e_{\text{eviction}}$ from \eqref{eq:eviction_error}. 

\textbf{(Case 2) Multiple attention heads}: this is trivial and can be achieved by replacing $\Delta_{A} = \text{concat}(\Delta_{A}^{1}, \dots, \Delta_{A}^{h})$, where super-script denotes head number. 

\textbf{(Case 3) Single self-attention and feedforward-network (FFN)}: we still assume single head without layer-norms. The dense logit is given as
\begin{align}
    X^{1/2}_{b+1} =& X^{0}_{b+1} + W_{O}X^{0}_{A,b+1}
    \\
    X^{1}_{b+1} =& X^{1/2}_{b+1} + W_{FFN}X^{1/2}_{b+1}
    \\
                =& X^{0}_{b+1} + W_{O}X^{0}_{A,b+1} \nonumber \\ 
                & + W_{FFN}X^{0}_{b+1} + W_{FFN}W_{O}X^{0}_{A,b+1}
    \\
    l_{b+1} =& W_{H}X_{b+1}^{1}
\end{align}
where, for simplicity we assume feedforward network to subsumed within $W_{FFN}$. 

The perturbed logit due to eviction is given as:
\begin{align}
    \hat{X}^{1/2}_{b+1} =& X^{0}_{b+1} + W_{O}X_{A,b+1}^{0} + \textcolor{red}{W_{O}\Delta_{A}^{0}}
    \\
    \hat{X}^{1}_{b+1} =& \hat{X}^{1/2}_{b+1} + W_{FFN}\hat{X}^{1/2}_{b+1}
    \\
                    =& X^{0}_{b+1} + W_{O}X_{A,b+1}^{0} + \textcolor{red}{W_{O}\Delta_{A}^{0}} \nonumber \\
                    & + W_{FFN}X^{0}_{b+1} + W_{FFN}W_{O}X_{A,b+1}^{0} \nonumber \\
                    & + \textcolor{red}{W_{FFN}W_{O}\Delta_{A}^{0}}
    \\
    \hat{l}_{b+1} =& W_{H}\hat{X}_{b+1}^{1} 
    \\
                    =& l_{b+1} + \textcolor{red}{\Delta_{l,b+1}}
\end{align}
where, the logit error $\Delta_{l,b+1} = W_{H}(W_{O}\Delta_{A}^{0} + W_{FFN}W_{O}\Delta_{A}^{0})$. Thus, the above analysis shows that error in attention output can cause discrepancy in logit space which can affect performance on downstream tasks.

Note that for multiple layers, each layer would have its own eviction error which will keep compounding; however, this computing the exact compounded error is non-trivial due to the presence of output layer-norms.

\section{Relation between \textit{CAOTE} and \textit{FastCAOTE}}
\begin{table*}[]
    \centering
    \resizebox{\textwidth}{!}{%
    \begin{tabular}{|*{15}{c|}}
        \hline
        Layer & 1 & 2 & 3 & 4 & 5 & 6 & 7 & 8 & 9 & 10 & 11 & 12 & 13 & 14 \\
        \hline
        & 0.99 & 0.98 & 0.91 & 0.97 & 0.92 & 0.88 & 0.85 & 0.82 & 0.81 & 0.90 & 0.84 & 0.92 & 0.83 & 0.86 \\
        \hline
    \end{tabular}
    }
    
    \vspace{0.5cm}
    
    \resizebox{\textwidth}{!}{%
    \begin{tabular}{|*{15}{c|}}
        \hline
        layer & 15 & 16 & 17 & 18 & 19 & 20 & 21 & 22 & 23 & 24 & 25 & 26 & 27 & 28 \\
        \hline
        & 0.91 & 0.88 & 0.94 & 0.96 & 0.92 & 0.95 & 0.99 & 0.99 & 0.96 & 0.98 & 0.98 & 0.88 & 0.99 & 0.96 \\
        \hline
    \end{tabular}
    }
    \caption{Spearman's correlation between \textit{CAOTE} and \textit{FastCAOTE}  on \textbf{Llama3.2-3B-Instruct} using sample from NarrativeQA}
    \label{tab:correlation_caote_fast_caote}
\end{table*}
We emperically show in \cref{tab:correlation_caote_fast_caote} that \textit{CAOTE} and \textit{FastCAOTE} are correlated, therefore, showing the efficacy of \textit{FastCAOTE}

\section{Additional Results}
\label{sec:appendix_results}

\subsection{Task context and generation lengths}
\label{subsec:task_gen_len}
\begin{table}[htbp]
    \centering
    \caption{Task Input Sizes}
    \begin{tabular}{lc}
        \hline
        \textbf{Task Type} & \textbf{Input Size (tokens)} \\
        \hline
        QA tasks & 4k -- 16k \\
        Summarization & 2.1k -- 16k \\
        Few-shot learning & 5.2k -- 22.4k \\
        Synthetic tasks & 7k -- 11k \\
        Code & 1.5k -- 4.2k \\
        Needle-in-the-haystack & up to 32k \\
        \hline
    \end{tabular}
    \label{tab:task_sizes}
\end{table}

\begin{table}[htbp]
    \centering
    \caption{Generation Lengths}
    \begin{tabular}{lc}
        \hline
        \textbf{Task Type} & \textbf{Generation Length (tokens)} \\
        \hline
        QA tasks & 100 -- 300 \\
        Summarization & 300 -- 800 \\
        Multi-document reasoning & 500 -- 1000 \\
        Code & $\sim$500 \\
        \hline
    \end{tabular}
    \label{tab:generation_lengths}
\end{table}
We show the average context and generation lengths of different tasks in \cref{tab:task_sizes}, \ref{tab:generation_lengths} respectively. 

\subsection{LongBench}
\label{subsec:appendix_results_longbench}
\begin{table*}[!t]
\caption{\textbf{LongBench results for Llama 3.1-8B and Llama 3.2-3B-Instruct. }
Higher number is better. 
We highlight the best performing methods within a given budget with \textbf{bold} and the second best with underline. 
} 
\label{table:longbench_reduced_llama_68}
\vspace{-5mm}
\begin{center}
\resizebox{\textwidth}{!}{
\begin{tabular}{clccccccccccccccccc}
\toprule

&
&
\multicolumn{3}{c}{Single Doc. QA} &
\multicolumn{3}{c}{Multi Doc. QA} &
\multicolumn{3}{c}{Summarization} &
\multicolumn{3}{c}{Few$\-$shot Learning} &
\multicolumn{2}{c}{Synthetic} &
\multicolumn{2}{c}{Code} &
\\

\cmidrule(lr){3-5}
\cmidrule(lr){6-8}
\cmidrule(lr){9-11}
\cmidrule(lr){12-14}
\cmidrule(lr){15-16}
\cmidrule(lr){17-18}

&
& 
{Narrative QA} & {Qasper} & {MF-en} & 
{HotpotQA} & {2WikiMQA} & {Musique} & 
{GovReport} & {QMSum} & {MultiNews} & 
{TREC} & {TriviaQA} & {SAMSum} & 
{PCount} & {PR-en} &
{Lcc} & {RB-P} & 
{Avg.} 
\\

\midrule

\multicolumn{2}{c}{Llama 3.1-8B} &
30.05
& 47.00 & 56.12 & 57.33 & 47.81 & 32.25 & 34.86 & 25.32 & 27.02 & 73.00 & 91.61 & 43.37 & 8.33 & 99.50 & 61.66 & 51.94 & 49.20 \\
\midrule

\multirow{11}{*}{6k} & Sink & 25.41 & 47.40 & 44.13 & 47.39 & 45.73 & 21.90 & 32.53 & 22.19 & 26.87 & 72.00 & 91.25 & 43.41 & 3.08 & 52.50 & 62.22 & 56.24 & 43.39
\\
\cmidrule(lr){2-19}

 &
H2O & 
8.52 & 43.31 & 44.80 & 40.03 & 42.46 & 21.68 & 11.85 & 8.78 & 26.03 & 62.00 & 56.39 & 25.72 & 5.75 & 45.50 & 58.62 & 29.53 & 33.19 \\
\cmidrule(lr){2-19}

& 
\multicolumn{1}{r}{+ CAOTE} & 
25.77 & 46.45 & 54.99 & 50.57 & 47.7 & 31.93 & 33.99 & 22.86 & 27.01 &  71 & 87.05 & 41.29 & 6.67 & 54 & 60.48 & 43.23 & 44.06
\\

& 
\multicolumn{1}{r}{+ FastCAOTE} & 
27.02 & 46.46 & 55.4 & 51.32 & 47.4 & 32.89 & 34.08 & 23.69 & 27.03 & 71 & 86.25 & 42.14 & 9 & 48.5 & 60.49 & 42.94 & 44.10 
\\

\cmidrule[0.8pt](lr){2-19}

&
TOVA & 
24.59 & 45.93 & 53.92 & 55.09 & 47.43 & 25.07 & 32.33 & 24.10 & 27.00 & 68.50 & 90.81 & 43.89 & 4.25 & 67.00 & 61.50 & 52.39 & 45.24 \\
\cmidrule(lr){2-19}

& 
\multicolumn{1}{r}{+ CAOTE} &
24.23 & 45.88 & 53.5 & 52.96 & 49.59 & 27.02 & 32.62 & 23.86 & 27.08 & 70 & 90.98 & 43.45 & 3 & 74.5 & 61.46 & 51.76 & 45.74 \\

& 
\multicolumn{1}{r}{+ FastCAOTE}& 
24.17 & 46.07 & 53.8 & 53.53 & 48.11 & 26.49 & 32.64 & 23.88 & 27.01 &  70 & 90.81 & 43.53 & 3 & 73  & 61.49 & 51.43 & 45.56 \\

\cmidrule[0.8pt](lr){2-19}

&
SnapKV & 
24.10 & 45.57 & 50.44 & 53.12 & 48.41 & 24.27 & 33.43 & 23.53 & 27.03 & 71.50 & 92.28 & 43.58 & 5.25 & 98.00 & 61.32 & 52.16 & 47.12 \\
\cmidrule(lr){2-19}

& 
\multicolumn{1}{r}{+ CAOTE} &
25.97 & 46.09 & 51.54 & 55.19 & 47.41 & 26.48 & 33.32 & 24 & 27.05 & 71.5 & 91.11 & 43.55 & 6.83 & 99.5 & 61.11 & 51.45 & \underline{47.63} \\

& 
\multicolumn{1}{r}{+ FastCAOTE} & 
24.77 & 46.06 & 52.18 & 56.72 & 47.01 & 26.24 & 33.41 & 23.8 & 26.99 & 73 & 91.31 & 43.6 & 5.92 & 99.5 & 61.5 & 51.37 & \textbf{47.71} \\

\midrule
\multirow{11}{*}{8k} & Sink & 23.53 & 46.63 & 48.68 & 49.61 & 47.16 & 21.14 & 33.10 & 23.20 & 26.92 & 72.00 & 91.29 & 43.79 & 3.25 & 66.00 & 62.18 & 56.43 & 44.68
\\
\cmidrule(lr){2-19}
 &
H2O & 
13.85 & 44.94 & 47.81 & 43.64 & 44.90 & 23.65 & 18.78 & 11.35 & 26.49 & 69.50 & 69.05 & 33.41 & 5.25 & 62.50 & 59.74 & 36.26 & 38.20
\\
\cmidrule(lr){2-19}

& 
\multicolumn{1}{r}{+ CAOTE} & 
27.74 & 46.67 & 54.97 & 52.71 & 48.28 & 33.66 & 34.51 & 24.73 & 26.99 & 73 & 86.8 & 42.86 & 5 & 66.5 & 61.06 & 48.29 & 45.86 
\\

& 
\multicolumn{1}{r}{+ FastCAOTE} & 
28.8 & 47.08 & 54.67 & 52.95 & 47.13 & 34.36 & 34.21 & 24.53 & 27.04 & 72 & 87.87 & 43.03 & 5.5 & 69.5 & 61.06 & 49.69 & 46.21 
\\

\cmidrule[0.8pt](lr){2-19}

&
TOVA &  
24.86 & 46.78 & 54.83 & 54.52 & 49.00 & 26.40 & 33.44 & 24.76 & 27.00 & 71.00 & 91.11 & 43.29 & 6.25 & 87.00 & 61.49 & 51.79 & 47.09
\\
\cmidrule(lr){2-19}

& 
\multicolumn{1}{r}{+ CAOTE} &
25.65 & 46.88 & 54.5 & 55.42 & 48.73 & 26.54 & 33.47 & 24.8 & 27.02 & 72 & 91.11 & 43.26 & 6.25 & 87 & 61.36 & 51.3 & 47.21
\\

& 
\multicolumn{1}{r}{+ FastCAOTE}& 
25.25 & 46.75 & 54.76 & 56.29 & 48.94 & 26.25 & 33.37 & 24.81 & 27.01 & 71.5 & 91.11 & 43.24 & 5.25 & 89 & 61.32 & 52.1  & \textbf{48.58}
\\ 

\cmidrule[0.8pt](lr){2-19}
& SnapKV &
25.15 & 46.55 & 53.39 & 56.00 & 48.75 & 27.82 & 33.67 & 24.85 & 27.01 & 72.50 & 91.78 & 43.54 & 5.08 & 100.00 & 61.48 & 51.41 & 48.06
\\
\cmidrule(lr){2-19}

& 
\multicolumn{1}{r}{+ CAOTE} &
 27.06 & 46.42 & 53.76 & 56.51 & 47.79 & 28.13 & 33.87 & 24.93 & 27.02 & 73 & 91.38 & 43.29 & 6.75 & 99.5 & 61.49 & 51.99 & 48.31
\\

& 
\multicolumn{1}{r}{+ FastCAOTE} & 
 26.91 & 46.59 & 53.47 & 56.63 & 48.54 & 29.27 & 33.91 & 24.86 & 27.01 & 73 & 91.38 & 43.49 & 6.75 & 100 & 61.49 & 51.78 & \underline{48.44}
\\


\midrule
\midrule
\addlinespace
\multicolumn{2}{c}{Llama 3.2-3B} &
23.76 & 40.23 & 50.09 & 50.69 & 42.29 & 26.84 & 33.09 & 24.30 & 25.21 & 72.50 & 90.11 & 42.58 & 3.00 & 96.50 & 56.22 & 56.52 & 45.87 \\
\midrule
\multirow{11}{*}{6k} & Sink &
19.33 & 40.29 & 37.95 & 46.48 & 40.29 & 15.31 & 30.43 & 21.35 & 25.14 & 71.50 & 88.93 & 42.04 & 3.50 & 47.00 & 56.55 & 54.11 & 40.01
\\
\cmidrule[0.8pt](lr){2-19}
 & H2O & 4.62 & 38.81 & 39.06 & 34.66 & 35.52 & 15.21 & 10.51 & 10.01 & 24.25 & 61.50 & 53.23 & 27.37 & 0.50 & 13.00 & 54.55 & 32.29 & 28.44
\\
\cmidrule[0.8pt](lr){2-19}
& \multicolumn{1}{r}{+CAOTE} & 16.14 & 41.68 & 49.36  & 46.7 & 43.36 & 22.75  & 32.07 & 21 & 25.07  & 69 & 80.02 & 39.33 & 1.5 & 26 & 55.82 & 49.05 & 38.68
\\
& \multicolumn{1}{r}{+FastCAOTE} &  16.31 & 41.94 & 49.17  & 45.64 & 41.83 & 21.68 & 32.07 & 20.73 & 25.02 & 68.5 & 80.34 & 39.88 & 3.5 & 24 & 55.83 & 48.7 & 38.45
\\
\cmidrule[0.8pt](lr){2-19}
& TOVA & 20.22 & 39.78 & 45.86 & 49.08 & 41.54 & 20.43 & 30.50 & 22.17 & 25.11 & 66.50 & 89.00 & 42.50 & 4.00 & 46.50 & 55.57 & 57.53 & 41.02
\\
\cmidrule[0.8pt](lr){2-19}
& \multicolumn{1}{r}{+CAOTE} & 21.17 & 39.69 & 47.21 & 48.82 & 41.7 & 20.59 & 30.72 & 22.36 & 25.1 &68 & 89 & 42.38 & 3.5 & 52.5 & 55.6 & 57.09 & 41.59
\\
& \multicolumn{1}{r}{+FastCAOTE} & 21.48 & 39.66 & 47.02 & 47.56 & 41.95 & 19.91 & 30.8 & 21.98 & 25.17 & 67.5 & 89.5 & 42.06  & 4 & 53 & 55.6 & 57.39 & 41.54
\\
\cmidrule[0.8pt](lr){2-19}
& SnapKV & 20.83 & 39.65 & 44.48 & 49.30 & 40.18 & 20.28 & 31.27 & 22.73 & 25.09 & 69.00 & 89.95 & 41.47 & 4.00 & 85.00 & 55.69 & 57.82 & 43.55
\\
\cmidrule[0.8pt](lr){2-19}
& \multicolumn{1}{r}{+CAOTE} & 20.23 & 39.65 & 44.91 & 50.16 & 40.58 & 21.32 & 31.23 & 22.51 & 25.13 & 69 & 90 & 41.83 & 5 & 89.5 & 55.84 & 57.24 & \underline{44.01}
\\
& \multicolumn{1}{r}{+FastCAOTE} & 20.09 & 40.02 & 44.58  & 48.57 & 42.12 & 22.51 & 31.25 & 22.89 & 25.15& 71 & 90 & 41.83  & 4 & 89.5 & 55.83 & 57.25 & \textbf{44.16}
\\

\midrule
\multirow{11}{*}{8k} & Sink &
20.15 & 40.02 & 41.94 & 48.15 & 42.24 & 16.01 & 31.64 & 22.10 & 25.20 & 73.00 & 89.26 & 42.37 & 3.50 & 62.50 & 56.86 & 56.63 & 41.97
\\
\cmidrule[0.8pt](lr){2-19}
 & H2O & 
9.65 & 39.66 & 43.20 & 38.09 & 40.41 & 21.46 & 17.80 & 13.28 & 24.67 & 70.00 & 64.30 & 32.19 & 2.00 & 24.50 & 55.00 & 39.09 & 33.46
\\
\cmidrule[0.8pt](lr){2-19}
& \multicolumn{1}{r}{+CAOTE} &
20.07 & 40.73 & 47.76  & 47.25 & 42.88 & 23.19  & 32.41 & 22.01 & 25.15  & 71 & 83.58 & 40.8 & 3 & 43.5  & 55.45 & 53.35 & 40.76
\\
& \multicolumn{1}{r}{+FastCAOTE} & 
 20.81 & 40.54 & 48.1  & 47.35 & 43.4 & 25.13  & 32.73 & 22.31 & 25.18 & 71.5 & 84.91 & 40.6 & 4 & 45 & 55.84 & 52.89 & 41.27
\\
\cmidrule[0.8pt](lr){2-19}
& TOVA & 21.08 & 40.67 & 49.07 & 48.69 & 41.93 & 23.05 & 31.64 & 22.85 & 25.21 & 69.00 & 89.25 & 42.19 & 2.50 & 71.00 & 55.77 & 57.47 & 43.21
\\
\cmidrule[0.8pt](lr){2-19}
& \multicolumn{1}{r}{+CAOTE} &
21.97 & 40.66 & 49.37 & 50.1 & 41.29 & 24.05 & 31.65 & 22.85 & 25.16 & 69.5 & 89.5 & 42 & 3 & 78.5  & 55.82 & 57.16 & 43.91
\\
& \multicolumn{1}{r}{+FastCAOTE} &
22.73 & 40.51 & 49.36 & 50.18 & 42.26 & 24.45 & 31.68 & 23.09 & 25.16 & 69.5 & 89.5 & 42.28 & 4.5 & 80 & 55.79 & 57.16 & 44.26
\\
\cmidrule[0.8pt](lr){2-19}
& SnapKV & 20.49 & 40.80 & 48.16 & 48.78 & 41.65 & 24.79 & 31.81 & 23.46 & 25.17 & 70.00 & 90.17 & 41.99 & 5.00 & 94.00 & 55.77 & 57.29 & \textbf{44.96}
\\
\cmidrule[0.8pt](lr){2-19}
& \multicolumn{1}{r}{+CAOTE} &
19.71 & 40.7 & 48.05 & 49.03 & 41.27 & 22.95 & 31.95 & 23.1 & 25.21 & 72 & 90 & 41.88 & 4 & 95  & 55.77 & 57.02 & \underline{44.85}
\\
& \multicolumn{1}{r}{+FastCAOTE} & 20.13 & 40.71 & 48.35  & 48.62 & 41.04 & 24.38 & 32.19 & 23.04 & 25.20 & 72 & 90 & 42.33 & 3.5 & 95  & 55.77 & 57.03 & \textbf{44.96}
\\
\bottomrule
\end{tabular}
}
\end{center}
\end{table*}
\begin{table*}[!t]
\caption{\textbf{LongBench results for Qwen 2.5-7B/2.5-3B-Instruct.}
Higher number is better. 
We highlight the best performing methods within a given budget with \textbf{bold} and the second best with underline. 
} 
\label{table:longbench_reduced_qwen_68}
\vspace{-5mm}
\begin{center}
\resizebox{\textwidth}{!}{
\begin{tabular}{clccccccccccccccccc}
\toprule

&
&
\multicolumn{3}{c}{Single Doc. QA} &
\multicolumn{3}{c}{Multi Doc. QA} &
\multicolumn{3}{c}{Summarization} &
\multicolumn{3}{c}{Few$\-$shot Learning} &
\multicolumn{2}{c}{Synthetic} &
\multicolumn{2}{c}{Code} &
\\

\cmidrule(lr){3-5}
\cmidrule(lr){6-8}
\cmidrule(lr){9-11}
\cmidrule(lr){12-14}
\cmidrule(lr){15-16}
\cmidrule(lr){17-18}

&
& 
{Narrative QA} & {Qasper} & {MF-en} & 
{HotpotQA} & {2WikiMQA} & {Musique} & 
{GovReport} & {QMSum} & {MultiNews} & 
{TREC} & {TriviaQA} & {SAMSum} & 
{PCount} & {PR-en} &
{Lcc} & {RB-P} & 
{Avg.} 
\\

\midrule

\multicolumn{2}{c}{Qwen 2.5-7B} &
 15.75 & 16.94 & 32.38 & 11.89 & 11.88 & 7.95 & 34.33 & 19.91 & 22.67 & 65.5 & 87.05 & 44.75 & 4.22 & 93.08 & 57.74 & 61.84 & 36.74
 \\
\midrule
\multirow{11}{*}{6k} & Sink & 7.37 & 16.61 & 25.73 & 11.29 & 11.27 & 5.69 & 31.47 & 18.72 & 22.86 & 64.5 & 84.86 & 44.47 & 3.59 & 41.48 & 55.89 & 55.99 & 31.36
\\
\cmidrule[0.8pt](lr){2-19}

 & H2O & 3.34 & 14.79 & 23.94 & 11.45 & 11.3 & 5.52 & 14.63 & 14.27 & 22.06 & 55.75 & 51.99 & 28.01 & 1.39 & 9.41 & 54.68 & 38.32 & 22.55

\\
\cmidrule(lr){2-19}
& 
\multicolumn{1}{r}{+ CAOTE} & 4.78 & 18.06 & 32.49 & 16.23 & 17.28 & 9.57 & 29.81 & 18.04 & 22.86 & 59.5 & 63.05 & 36.91 & 2.7 & 28.25 & 55.13 & 42.42 & 28.57

\\
& 
\multicolumn{1}{r}{+ FastCAOTE} & 5.69 & 16.99 & 32.62 & 18.22 & 16.58 & 10.48 & 30.3 & 17.71 & 22.88 & 59.5 & 62.95 & 36.29 & 2.1 & 27.65 & 56.3 & 40.65 & 28.56

\\

\cmidrule[0.8pt](lr){2-19}

&
TOVA & 15.77 & 15.33 & 30.31 & 19.3 & 13.78 & 9.11 & 30.4 & 19.95 & 22.91 & 61.5 & 83.47 & 42.9 & 1.15 & 21.775 & 57.68 & 57.99 & 31.46
 
\\
\cmidrule(lr){2-19}

& 
\multicolumn{1}{r}{+ CAOTE} & 15.81 & 16.07 & 29.39 & 19.4 & 14.15 & 10.8 & 30.89 & 20.54 & 22.86 & 62 & 84.92 & 43.19 & 2.17 & 30 & 57.76 & 57.53 & 32.34

\\

& 
\multicolumn{1}{r}{+ FastCAOTE}& 15.67 & 16.23 & 30.4 & 19.45 & 13.32 & 10.18 & 30.77 & 20.2 & 22.82 & 61.5 & 83.29 & 43.3 & 1.5 & 28.75 & 57.71 & 58.1 & 32.07

\\

\cmidrule[0.8pt](lr){2-19}

&
SnapKV & 14.34 & 16.35 & 31.12 & 17.56 & 14.1 & 8.74 & 31.09 & 20.16 & 22.84 & 60 & 83.8 & 42.99 & 2.91 & 54.17 & 57.48 & 60.26 & 33.62

\\
\cmidrule(lr){2-19}

& 
\multicolumn{1}{r}{+ CAOTE} & 12.77 & 16.3 & 31.33 & 19.74 & 14.06 & 11.07 & 31.02 & 20.85 & 22.91 & 61.5 & 83.79 & 42.97 & 4.8 & 68.25 & 57.54 & 61.08 & \textbf{35.00}

 \\

& 
\multicolumn{1}{r}{+ FastCAOTE} & 12.98 & 15.93 & 31.3 & 18.58 & 13.82 & 9.45 & 30.96 & 20.27 & 22.88 & 61.5 & 84.58 & 43.28 & 5.34 & 65.48 & 57.54 & 60.25 & \underline{34.63}

 \\


\midrule

\multirow{10}{*}{8k} &
H2O & 6.1 & 15.55 & 28.29 & 12.37 & 14.65 & 6.24 & 20.78 & 17.22 & 22.44 & 59 & 58.74 & 33.05 & 1.82 & 15.73 & 55.63 & 44.56 & 25.76

\\
\cmidrule(lr){2-19}

& 
\multicolumn{1}{r}{+ CAOTE} & 
8.65 & 15.59 & 34.92 & 20.41 & 15.95 & 13.6 & 32.11 & 20.05 & 22.82 & 63.5 & 78.2 & 40.66 & 3.85 & 46.33 & 57.19 & 51.7 & 32.85
\\

& 
\multicolumn{1}{r}{+ FastCAOTE} & 6.76 & 15.88 & 34.3 & 20.75 & 16.2 & 16.82 & 31.95 & 20.67 & 22.81 & 62.5 & 77.33 & 41.02 & 3.03 & 47.08 & 57.23 & 50.48 & 32.8
\\ 

\cmidrule[0.8pt](lr){2-19}

&
TOVA & 15.69 & 15.55 & 33.09 & 18.37 & 13.99 & 11.26 & 31.33 & 20.17 & 22.82 & 62 & 84.49 & 43.01 & 2.78 & 30.33 & 57.45 & 58.96 & 32.58

\\
\cmidrule(lr){2-19}

& 
\multicolumn{1}{r}{+ CAOTE} & 16.38 & 15.46 & 32.16 & 17.86 & 14.24 & 12.76 & 31.34 & 20.2 & 22.8 & 61 & 83.97 & 43.23 & 2.01 & 38.83 & 57.45 & 59.41 & 33.07

\\ 

& 
\multicolumn{1}{r}{+ FastCAOTE}& 17.06 & 15.55 & 32.32 & 17.57 & 14.14 & 13.03 & 31.43 & 20.3 & 22.78 & 62 & 84.86 & 43.3 & 2.41 & 41.58 & 57.41 & 58.93 & 33.42

\\

\cmidrule[0.8pt](lr){2-19}

&
SnapKV &  15.6 & 15.81 & 33.47 & 18.02 & 14.49 & 10.53 & 31.99 & 20.09 & 22.84 & 61 & 84.08 & 43.01 & 4.58 & 64.25 & 57.46 & 60.59 & 34.86

\\
\cmidrule(lr){2-19}

& 
\multicolumn{1}{r}{+ CAOTE} & 15.55 & 15.57 & 33.89 & 21.08 & 14.43 & 12.38 & 31.41 & 20.73 & 22.83 & 61.5 & 85.11 & 43.39 & 5.22 & 75.75 & 57.44 & 60.35 & \textbf{36.04}

\\

& 
\multicolumn{1}{r}{+ FastCAOTE} & 13.38 & 15.77 & 33.97 & 19.78 & 15.08 & 13.01 & 31.44 & 20.69 & 22.77 & 62 & 85.66 & 43.69 & 4.24 & 75.33 & 57.44 & 60.04 & \underline{35.89}

\\

\midrule
\midrule
\addlinespace
\multicolumn{2}{c}{Qwen 2.5-3B} &
18.08 & 22.49 & 39.72 & 27.86 & 20.45 & 18.93 & 32.8 & 23.74 & 24.89 & 67.5 & 85.05 & 43.88 & 5 & 40.97 & 51.91 & 47.53 & 35.68
\\
\midrule
\multirow{11}{*}{6k} & Sink & 13.01 & 20.03 & 32.59 & 18.62 & 15.77 & 9.37 & 30.98 & 20.7 & 24.97 & 66.5 & 75.39 & 42.77 & 4 & 14.92 & 52.32 & 50.35 & 30.77
\\
\cmidrule[0.8pt](lr){2-19}

 & H2O &  5.52 & 18.62 & 27.93 & 12.61 & 15.07 & 4.26 & 14.92 & 13.89 & 24.21 & 58 & 45.94 & 24.93 & 2.91 & 9.1 & 49.5 & 34.54 & 22.62
\\
\cmidrule[0.8pt](lr){2-19}
& 
\multicolumn{1}{r}{+ CAOTE} & 8.23 & 21.34 & 36.28 & 22.43 & 17.92 & 13.53 & 31.57 & 21.2 & 24.91 & 65 & 74.3 & 39.09 & 4.58 & 21.25 & 50.47 & 42.71 & 30.93
\\
& 
\multicolumn{1}{r}{+ FastCAOTE} &  9.29 & 20.47 & 35.8 & 21.67 & 18.14 & 13.65 & 31.34 & 20.52 & 24.82 & 64.5 & 75.88 & 39.16 & 5.72 & 20.42 & 50.6 & 44.38 & 31.02
\\
\cmidrule[0.8pt](lr){2-19}
& TOVA & 13.62 & 19.56 & 34.64 & 21.67 & 16.25 & 8.47 & 30.17 & 23.1 & 24.94 & 63.5 & 81.88 & 42.97 & 1.16 & 10.58 & 51.3 & 47.7 & 30.72
\\
\cmidrule[0.8pt](lr){2-19}
&
\multicolumn{1}{r}{+ CAOTE} & 13.28 & 19.82 & 35.25 & 22.6 & 15.5 & 8.97 & 30.4 & 23.17 & 24.84 & 64.5 & 81.68 & 43.46 & 2.07 & 13.21 & 50.56 & 47.05 & 31.02
\\
& 
\multicolumn{1}{r}{+ FastCAOTE} & 13.34 & 19.71 & 35.58 & 22.02 & 15.65 & 8.76 & 30.49 & 23.26 & 24.82 & 65 & 80.92 & 43.66 & 1.75 & 13 & 51.45 & 47.74 & 31.07
\\
\cmidrule[0.8pt](lr){2-19}
& SnapKV & 14.16 & 20.09 & 36.15 & 19.14 & 15.59 & 12.7 & 30.35 & 22.75 & 24.91 & 65 & 83.92 & 43.52 & 5.00 & 32.2 & 51.04 & 47.49 & 32.75
\\
\cmidrule[0.8pt](lr){2-19}
& \multicolumn{1}{r}{+CAOTE} &  14.3 & 20.13 & 34.89 & 20.32 & 15.06 & 12.85 & 30.61 & 23.16 & 24.9 & 66.5 & 84.75 & 43.3 & 4.75 & 33.9 & 51.48 & 48.31 & \textbf{33.08}
\\
& \multicolumn{1}{r}{+FastCAOTE} & 14.33 & 19.48 & 34.59 & 20.8 & 16.54 & 11.85 & 30.68 & 23.19 & 24.93 & 66.5 & 83.54 & 43.55 & 4.62 & 33 & 51.06 & 47.93 & \underline{32.91}
\\
\midrule
\multirow{10}{*}{8k} & H2O & 6.16 & 19.84 & 32.32 & 16.01 & 17.74 & 4.99 & 20.21 & 16.49 & 24.54 & 64 & 56.1 & 32.56 & 3.13 & 11.61 & 50.61 & 38.8 & 25.94

\\
\cmidrule[0.8pt](lr){2-19}
& \multicolumn{1}{r}{+CAOTE} &  11.53 & 21.59 & 38.02 & 25.62 & 20.19 & 15.11 & 32.18 & 21.82 & 24.81 & 67.5 & 79.32 & 41.2 & 5.15 & 25.5 & 50.72 & 45.27 & 32.85

\\
& \multicolumn{1}{r}{+FastCAOTE} &  11.65 & 21.2 & 37.92 & 24.47 & 20.32 & 13.25 & 32.11 & 21.75 & 24.84 & 67.5 & 78.07 & 40.27 & 5.17 & 25.21 & 50.17 & 46.37 & 32.52

\\
\cmidrule[0.8pt](lr){2-19}
& TOVA & 14.66 & 20.93 & 37.77 & 22.57 & 17.08 & 9.63 & 31.12 & 23.17 & 24.83 & 67 & 84.11 & 43.55 & 2.06 & 13.08 & 51.32 & 47.64 & 31.91

\\
\cmidrule[0.8pt](lr){2-19}
& \multicolumn{1}{r}{+CAOTE} & 13.98 & 21 & 36.91 & 22.97 & 17.05 & 9.93 & 31.25 & 23.45 & 24.9 & 66.5 & 84.14 & 43.86 & 3.07 & 13.92 & 51.14 & 47.883 & 32.00

\\
& \multicolumn{1}{r}{+FastCAOTE} & 14.84 & 20.66 & 37.45 & 23.05 & 17.07 & 10.16 & 31.2 & 23.47 & 24.85 & 66.5 & 84.01 & 43.41 & 3.31 & 13.25 & 51.19 & 48.11 & 32.03

\\
\cmidrule[0.8pt](lr){2-19}
& SnapKV & 12.76 & 20.88 & 37.1 & 22.49 & 18.19 & 13.83 & 31.33 & 23.37 & 24.8 & 65.5 & 84.88 & 44.49 & 5.2 & 35.83 & 51.31 & 47.82 & 33.74

\\
\cmidrule[0.8pt](lr){2-19}
& \multicolumn{1}{r}{+CAOTE} & 13.66 & 20.41 & 38.08 & 24.76 & 17.31 & 13.21 & 31.3 & 23.62 & 24.82 & 66.5 & 84.88 & 44.16 & 5.17 & 36.58 & 51.24 & 47.94 & \underline{33.98}

\\
& \multicolumn{1}{r}{+FastCAOTE} & 14.56 & 20.99 & 37.61 & 25.56 & 18.02 & 13.89 & 31.37 & 23.27 & 24.83 & 66.5 & 84.88 & 44.01 & 5 & 35.92 & 51.13 & 48.14 & \textbf{34.11}

\\
\bottomrule
\end{tabular}
}
\end{center}
\end{table*}
LongBench result for $6$k and $8$K for \texttt{Llama 3.2-3B-Instruct}/3.1-8B-Instruct and \texttt{Qwen 2.5-3B-Instruct}/8B-Instruct are shown in Table \cref{table:longbench_reduced_llama_68}, \cref{table:longbench_reduced_qwen_68} respectively.
We also include Sink attention results \cite{xiaoefficient} with budget of $6$k and $8$k. For \texttt{Llama 3.1-8B-Instruct}, \textit{TOVA-FastCAOTE} performs best for $6$k budget, while \textit{SnapKV-FastCAOTE} for $8$k budget. For \texttt{Llama 3.2-3B-Instruct} \textit{SnapKV-FastCAOTE} performs best for both $6$k and $8$k budget. On the other hand, for \texttt{Qwen 2.5-7B-Instruct}, \textit{SnapKV-CAOTE} performs the best for both $6$k and $8$k, and for \texttt{Qwen 2.5-3B-Instruct}, \textit{SnapKV-COATE} performs best for $6$k budget, and \textit{SnapKV-FastCAOTE} performs best for $8$k budget. Additionally, note that all baseline token eviction methods achieve boost in accuracy when using \textit{CAOTE} or \textit{FastCAOTE}.

\subsection{Perplexity}
\label{subsec:appendix_results_ppl}
\begin{table*}[t]
\vspace{-2mm}
\caption{\textbf{Perplexity difference between different eviction methods with dense baseline.} Lower is better. Negative entry in table means the method performs better than dense baseline. The PPL of Qwen 2.5-3B-Instruct and Qwen 2.5-7B-Instruct is 8.4547 and 7.3188 respectively.}
\vspace{-2mm}
\label{table:PPL_qwen_wide}
\begin{center}\resizebox{0.8\textwidth}{!}{%
\begin{tabular}{cccccccccc}

\toprule

\multirow{2}{*}{Budget} & \multicolumn{3}{c}{H2O} & \multicolumn{3}{c}{TOVA} & 
\multicolumn{3}{c}{SnapKV}
\\

\cmidrule[0.5pt](lr){2-4}
\cmidrule[0.5pt](lr){5-7}
\cmidrule[0.5pt](lr){8-10}

& & +CAOTE & +FastCAOTE & & +CAOTE & +FastCAOTE & & +CAOTE & +FastCAOTE
\\

\midrule[0.7pt]
\multicolumn{10}{c}{Qwen 2.5-7B-Instruct}
\\
\midrule

2k & 0.4253 & 0.4422 & 0.3917 & 0.0567 & 0.059 & 0.1077 & 0.0369 & 0.0987 & \textbf{0.0307}
\\ 

\midrule[0.7pt]
\multicolumn{10}{c}{Qwen 2.5-3B-Instruct}
\\
\midrule

2k & 0.2585 & 0.2168 & 0.2154 & 0.0603 & 0.0513 & 0.0507 & 0.0278 & 0.0199 & \textbf{0.0196}
\\ 
\bottomrule
\end{tabular}
}%
\vspace{-5mm}
\end{center}
\end{table*}
We show perplexity results for \texttt{Qwen2.5} models in \cref{table:PPL_qwen_wide} for budget of $2$k. \textit{SnapKV-FastCAOTE} performs best for both \texttt{Qwen 2.5-3B-Instruct} and 2.5-7B-Instruct, and using \textit{CAOTE}, all methods achieve improved perplexity.

\subsection{Needle in a Haystack}
\label{subsec:appendix_results_needle}
\begin{table*}[t]
\caption{\textbf{Needle-in-haystack accuracy} for Qwen 2.5-3B/2.5-7B-Instruct using baseline eviction methods with(out) \textit{CAOTE}. Higher is better, maximum accuracy is $1.0$.
}
\vspace{-2mm}
\label{table:needle_qwen_wide}
\begin{center}\resizebox{0.8\textwidth}{!}{
\begin{tabular}{cccccccccc}

\toprule
\multirow{2}{*}{Budget} & \multicolumn{3}{c}{H2O} &  \multicolumn{3}{c}{TOVA} &  \multicolumn{3}{c}{SnapKV} 
\\
\cmidrule(lr){2-4}
\cmidrule(lr){5-7}
\cmidrule(lr){8-10}
& & +CAOTE & +FastCAOTE & & +CAOTE & +FastCAOTE & & +CAOTE & +FastCAOTE
\\
\midrule
\multicolumn{10}{c}{Qwen 2.5-7B-Instruct}
\\
\midrule
6k & 0.206	& 0.312	& 0.3 & 0.292 &	0.292 &	0.286 &	0.32 & 0.33 & \textbf{0.332}
\\
\midrule
\multicolumn{10}{c}{Qwen 2.5-3B-Instruct}
\\
\midrule
6k &  0.212 & 0.288 & 0.27 & 0.282 & 0.286 & 0.288 & 0.304 & 0.324 & \textbf{0.336}

\\
\bottomrule
\end{tabular}
}
\vspace{-5mm}
\end{center}
\end{table*}
\cref{table:needle_qwen_wide} shows Needle-in-haystack results for \texttt{Qwen2.5} models for budget=$6$k. \textit{SnapKV-FastCAOTE} performs best for both \texttt{Qwen 2.5-3B-Instruct} and 2.5-7B-Instruct, and using \textit{CAOTE}, all methods achieve improved accurracy.

\end{document}